\documentclass[journal,draftcls,onecolumn,10pt]{IEEEtran}
\usepackage[bookmarks,colorlinks,breaklinks]{hyperref}  
\hypersetup{linkcolor=blue,citecolor=blue,filecolor=blue,urlcolor=blue} 
\usepackage{graphicx}
\usepackage{url,algorithm,algorithmic}
\usepackage{amsmath,amsfonts,amsthm,amssymb,xspace,bm, verbatim,dsfont}
\usepackage{graphicx}
\usepackage{url,algorithm,algorithmic}
\usepackage{lipsum}
\usepackage{thmtools}
\usepackage{thm-restate}
\usepackage{hyperref}
\usepackage{cleveref}
\usepackage{tikz}
\usetikzlibrary{decorations.pathreplacing,angles,quotes}
\usepackage{bbm}
\usepackage{todonotes}
\usepackage{subfig}
\usepackage{pgfplots}
  
\theoremstyle{plain}
\newtheorem{theorem}{Theorem}[section]
\newtheorem{lemma}[theorem]{Lemma}

\theoremstyle{definition}
\newcommand{\note}[1]{\marginpar{\tiny *note in TeX*}}
\newcommand{\prob}[1]{\mathbb{P}\left[#1\right]}

\newcommand{\expec}[1]{\mathbb{E}\left[#1\right]}
\newcommand{\normal}{\mathcal{N}(0,1)}
\newcommand{\gauss}{\mathcal{N}}

\newcommand{\rn}{\mathbb{R}^n}
\newcommand{\rnn}{\mathbb{R}^{n \times n}}
\newcommand{\rmm}{\mathbb{R}^{m \times m}}
\newcommand{\rmn}{\mathbb{R}^{m \times n}}
\newcommand{\reals}{\mathbb{R}}

\newcommand{\eye}{\mathbb{I}}
\newcommand{\I}{\mathbf{I}}
\newcommand{\ones}{\mathbbm{1}}

\newcommand{\vect}[1]{\mathbf{#1}}
\newcommand{\mat}[1]{\mathbf{#1}}

\newcommand{\iprod}[2]{\left\langle #1, #2 \right\rangle} 
\newcommand{\abs}[1]{\left|#1\right|}
\newcommand{\phase}[1]{\mathrm{Ph}\left(#1\right)}

\newcommand{\norm}[1]{{\left\lVert{#1}\right\rVert}}

\newcommand{\zeronorm}[1]{\left\| {#1} \right\|_0}

\newcommand{\twonorm}[1]{\left\| {#1} \right\|_2}

\newcommand{\sign}[1]{\operatorname{sign}\left(#1\right)}

\newcommand{\dist}[2]{\|#1-#2\|_2}
\newcommand{\distop}[2]{\mathrm{dist}\left(#1,#2\right)}

\newcommand{\supp}[1]{\operatorname{supp}(#1)}
\DeclareMathOperator*{\argmin}{argmin}

\newcommand{\card}[1]{\operatorname{card}(#1)}

\newcommand{\order}[1]{\mathcal{O}\left({#1}\right)}
\newcommand{\ordereps}[1]{\mathcal{O}_{\epsilon}\left({#1}\right)}
\newcommand{\rbrak}[1]{\left(#1\right)}
\newcommand{\sbrak}[1]{\left[#1\right]}
\newcommand{\cbrak}[1]{\left\{#1\right\}}
\newcommand{\sqfr}[2]{\sqrt{\frac{#1}{#2}}}
\newcommand{\frsq}[2]{\frac{#1}{\sqrt{#2}}}

\newcommand{\eps}{\epsilon}
\newcommand{\y}{\vect{y}}

\newcommand{\z}{\vect{z}}
\newcommand{\w}{\vect{w}}
\newcommand{\uu}{\vect{u}}
\newcommand{\rr}{\vect{r}}
\newcommand{\vv}{\vect{v}}

\newcommand{\xo}{\vect{x^*}}
\newcommand{\xin}{\vect{x^0}}
\newcommand{\x}{\vect{x}}
\newcommand{\xt}{\x^t}
\newcommand{\xtplus}{\x^{t+1}}

\newcommand{\xSmin}{\x_{S_{-}}^*}
\newcommand{\xSpls}{{\x_{S_{+}}^*}}
\newcommand{\xSbmin}{\x_{S_{b-}}^*}
\newcommand{\xSbpls}{{\x_{S_{b+}}^*}}
\newcommand{\xShat}{{\x_{\hat{S}}^*}}
\newcommand{\xSbhat}{{\x_{\hat{S_b}}^*}}

\newcommand{\ai}{{\vect{a}_i}}
\newcommand{\aiS}{{\ai_{S}}}
\newcommand{\aiShat}{\ai_{\hat{S}}}

\renewcommand{\P}{\mat{P}}
\newcommand{\A}{\mat{A}}
\newcommand{\D}{\mat{D}}
\newcommand{\M}{\mat{M}}

\ifCLASSINFOpdf
\else
\fi

\begin{document}
\title{Sample-Efficient Algorithms for Recovering Structured Signals from Magnitude-Only Measurements}

\author{Gauri~Jagatap and~Chinmay~Hegde% 

\thanks{The authors are with the Electrical and Computer Engineering Department at Iowa State University, Ames, IA 50010. Email: \{gauri,chinmay\}@iastate.edu. The authors thank Piotr Indyk and Thanh Nguyen for their helpful feedback. A conference version of this manuscript will appear in the Annual Conference on Neural Information Processing Systems (NIPS) in December 2017~\cite{jagatap2017phase}; this version has an expanded set of theoretical results as well as numerical experiments. This work is supported in part by the National Science Foundation under the grants CCF-1566281 and IIP-1632116.}}% <-this % stops a space

%\markboth{IEEE TRANSACTIONS ON INFORMATION THEORY,  VOL. xx, NO. y, month year}%
%{Shell \MakeLowercase{\textit{et al.}}: Bare Demo of IEEEtran.cls for IEEE Journals}
% The only time the second header will appear is for the odd numbered pages
% after the title page when using the twoside option.
% 
% *** Note that you probably will NOT want to include the author's ***
% *** name in the headers of peer review papers.                   ***
% You can use \ifCLASSOPTIONpeerreview for conditional compilation here if
% you desire.

%\IEEEpeerreviewmaketitle
\maketitle
\vspace{-0.3cm}
\begin{abstract}

We consider the problem of recovering a signal $\xo \in \rn$, from magnitude-only measurements, $y_i = \abs{\iprod{\ai}{\xo}} $ for $i=\{1,2,\ldots,m\}$. This is a stylized version of the classical \emph{phase retrieval problem}, and is a fundamental challenge in nano- and bio-imaging systems, astronomical imaging, and speech processing. It is well known that the above problem is ill-posed, and therefore some additional assumptions on the signal and/or the measurements are necessary. 

In this paper, we consider the case where the underlying signal $\xo$ is $s$-sparse. For this case, we develop a novel recovery algorithm that we call \emph{Compressive Phase Retrieval with Alternating Minimization}, or \emph{CoPRAM}. Our algorithm is simple and be obtained via a natural combination of the classical alternating minimization approach for phase retrieval with the CoSaMP algorithm for sparse recovery. Despite its simplicity, we prove that our algorithm achieves a sample complexity of $\order{s^2 \log n}$ with Gaussian measurements $\ai$, which matches the best known existing results; moreover, it also demonstrates linear convergence in theory and practice. An appealing feature of our algorithm is that it requires no extra tuning parameters other than the signal sparsity level $s$. Moreover, we show that our algorithm is robust to noise.

The quadratic dependence of sample complexity on the sparsity level is sub-optimal, and we demonstrate how to alleviate this via \emph{additional} assumptions beyond sparsity. First, we study the (practically) relevant case where the sorted coefficients of the underlying sparse signal exhibit a power law decay. In this scenario, we show that the CoPRAM algorithm achieves a sample complexity of $\order{s \log n}$, which is close to the information-theoretic limit. 

We then consider the case where the underlying signal $\xo$ arises from \emph{structured} sparsity models. We specifically examine the case of \emph{block-sparse} signals with uniform block size of $b$ and block sparsity $k=s/b$. For this problem, we design a recovery algorithm that we call \emph{Block CoPRAM} that further reduces the sample complexity to $\order{ks \log n}$. For sufficiently large block lengths of $b=\Theta(s)$, this bound equates to $\order{s \log n}$. 

To our knowledge, our approach constitutes the first family of \emph{linearly convergent} algorithms for signal recovery from magnitude-only Gaussian measurements that exhibit a sub-quadratic dependence on the signal sparsity level. 
\end{abstract}

% Note that keywords are not normally used for peerreview papers.

%\IEEEpeerreviewmaketitle

%Introduction
\section{Introduction}
\label{sec:intro}

\subsection{Motivation}

\IEEEPARstart{I}{n} this paper, we consider the problem of recovering a high-dimensional vector $\xo \in \rn $ from (possibly noisy) {\em magnitude-only} linear measurements (or samples). That is, for $\ai \in \rn$, if
\begin{align} 
\label{eqn:magnitude-measurements}
y_i ~ = ~ \abs{\iprod{\vect{a_i}}{\xo}}, \quad \text{for $i = 1,\ldots,m$},
\end{align}
then the task is to recover $\xo$ using the samples $\y$ and the matrix $\A = [\vect{a_1}\ \vect{a_2}\ \dots\ \vect{a_m}]^\top$. 

Problems of this kind arise in numerous scenarios in machine learning, imaging, and statistics. For example, the classical problem of \emph{phase retrieval} is encountered in imaging systems including diffraction imaging, X-ray crystallography, ptychography, and astronomical imaging \cite{shechtman2015phase,millane1990phase,maiden2009improved,harrison1993phase,miao2008extending}. For such systems, the physics of light acquisition are such that the optical sensors can only record the intensity of the light waves but not its phase. In terms of our setup, the vector $\xo$ would correspond to an image (possessing a resolution of $n$ pixels) and the measurements correspond to the magnitudes of its 2D Fourier transform coefficients. The goal is to stably recover the image $\xo$ from the measurements, ideally with as few observations (i.e., as small $m$ as possible).

Despite the prevalence of several heuristic approaches~\cite{GerchbergS72,Fienup1982,marchesini2007phase,nugent2003unique} to solve \eqref{eqn:magnitude-measurements}, it is generally accepted that \eqref{eqn:magnitude-measurements} is a very challenging nonlinear, ill-posed inverse problem both in theory and practice. Indeed, for generic $\vect{a}_i$ and $\xo$, one can show that \eqref{eqn:magnitude-measurements} is \emph{NP-hard} by reduction from certain well-known combinatorial problems~\cite{fickus2014phase}. Therefore, additional assumptions on the vector $\xo$ and/or the measurements $\ai$ are necessary. 

A recent line of breakthrough results~\cite{phaselift,CandesLS14,netrapalli} have provided provably efficient algorithmic procedures for the special case where the measurement vectors are randomly drawn from certain multi-variate probability distributions (such as i.i.d.\ Gaussian distributions). By convention, we will continue to term such methods as ``phase retrieval" algorithms. 
%The seminal paper by Netrapalli \emph{et al.}~\cite{netrapalli} provides the first rigorous justification of classical heuristics for phase retrieval based on alternating minimization. 
However, all these newer results require an ``overcomplete" set of observations, i.e., the number of observations $m$ exceeds the problem dimension $n$, sometimes by a significant amount. This requirement can pose severe limitations on computation, storage, and processing the measurements, particularly in the high-dimensional regime when $m$ and $n$ are very large. Our focus in this paper is to address the following:

\begingroup
\addtolength\leftmargini{0.3in}
\begin{quote}
\emph{\textbf{Challenge}: Can we solve the phase retrieval problem using very few samples (in particular, significantly fewer samples than the problem dimension)?}
\end{quote}
\endgroup

A possible solution to the above challenge is to leverage the fact that in many practical applications, $\xo$ often obeys certain \emph{low-dimensional} structural assumptions. A common structural assumption used in imaging applications is that $\xo$ is $s$-\emph{sparse} in some known basis, such as the identity or the wavelet basis. For transparency, we assume that the sparsity basis is the canonical basis throughout this paper, unless otherwise specified. %The true support $S \subset \{1,2,\dots n\}$. 
Similar structural assumptions form the core of sparse recovery, compressive sensing, and streaming algorithms~\cite{candes2006robust,needell2008greedy,candes2006stable}, and it has been established that only $\order{s \log \frac{n}{s}}$ samples are necessary for stable recovery of $\xo$; moreover, the dependence of the number of samples on $n$ and $s$ is information-theoretically optimal~\cite{khanhdobaindyk}.

Solving the sparsity-constrained version of \eqref{eqn:magnitude-measurements} (sometimes referred to as \emph{sparse phase retrieval} in the literature) is therefore a natural next step, and numerous approaches have been proposed in this regard. These include a variant of alternating minimization~\cite{netrapalli}, methods based on convex relaxation~\cite{cprl,chenchigoldsmith,oymak}, and iterative thresholding-based techniques~\cite{cai,sparta}. However, all existing methods suffer from one (or more) of the following drawbacks:
\begin{enumerate}
\item Somewhat curiously, all of the above algorithms incur a sample complexity of $\Omega(s^2 \log n)$ for stable recovery, which is \emph{quadratically worse} than the information-theoretic limit of $\order{s \log \frac{n}{s}}$\footnote{Exceptions to this rule are the approaches of \cite{iwen,bahmani,pal,cai2014super,yin2016compressed,pedarsani2017phasecode}, which indeed achieve near-optimal sample complexity and/or running time; however, these schemes are applicable only for very carefully designed measurements $a_i$.}.
\item Most algorithms suffer from a running time that is quadratic (or worse) in the dimension of the signal~\cite{cprl,cai}.
\item Many algorithms require stringent assumptions on the minimum (absolute) value of the nonzero signal coefficients~\cite{netrapalli,sparta}.
\item Typically, these algorithms require tuning of several parameters for their proper functioning~\cite{cprl,cai,sparta}.
\end{enumerate}

Finally, in specific applications, more refined structural assumptions on $\xo$ \emph{beyond sparsity} are applicable. For example, point sources in astronomical images often produce \emph{clusters} of nonzero pixels in a given image, while wavelet coefficients of natural images often can be organized as connected \emph{sub-trees}. Algorithms that leverage such \emph{structured sparsity} assumptions have been shown to achieve considerably improved sample-complexity in statistical learning and sparse recovery problems~\cite{grouplasso,modelcs,eldar2010block}. Indeed, a plethora of algorithms for modeling several types of structured sparsity constraints, including block-sparsity~\cite{modelcs,huang2011learning}, tree sparsity~\cite{MarcoCISS,modelcs,treesISIT,modelcsICALP}, clusters~\cite{ModelCSSAMPTA,huang2011learning,clustersICML}, and graph-based models~\cite{csmrf,clustersICML,approxSODA}. However, a systematic approach that leverage structured sparsity models in the context of phase retrieval does not seem to have been studied in the literature.

\subsection{Our contributions}

In this paper, we establish a flexible algorithmic framework that systematically leverages (structured) sparsity-based signal models for the phase retrieval problem. We rigorously show that our approach matches the best available state-of-the-art sparse phase retrieval methods both from a statistical as well as computational viewpoint. Next, we show that it is possible to extend this algorithm to the case where the signal obeys certain types of block-sparsity structures, thereby \emph{further} lowering the sample complexity of stable signal recovery. % close to the information theoretic limit. %Our work can be viewed as a first step towards phase retrieval of structured sparse signals using Gaussian measurements.
%\red{unify ``signal" and ``vector" everywhere}

\begin{enumerate}
\item We first consider the standard case where the underlying signal $\xo$ is $s$-sparse (i.e. it has underlying model $\mathcal{M}_s$, where $\mathcal{M}_s$ consists of all $s$-sparse vectors of dimension $n$). For this case, we develop a novel recovery algorithm that we call \emph{Compressive Phase Retrieval with Alternating Minimization}, or \emph{CoPRAM}. Our algorithm is simple and be obtained via a natural combination of the classical alternating minimization approach for phase retrieval with the CoSaMP algorithm for sparse recovery, together with a smart initialization. Despite its simplicity, we prove that our algorithm achieves a sample complexity of $\order{s^2 \log n}$ with Gaussian sampling vectors $\ai$ in order to achieve linear convergence, which matches the best among all available existing results. An appealing feature of our algorithm is that it requires no extra \emph{a priori} information other than the signal sparsity level $s$, and that it requires no assumptions whatsoever on the nonzero signal coefficients. Finally, we show that our algorithm is stable with respect to noise in the measurements. To our knowledge, this is the first algorithm for sparse phase retrieval that simultaneously achieves all of the above properties. \footnote{Note that we use the terms \textit{sparse phase retrieval} and \textit{compressive phase retrieval} interchangeably throughout the course of this paper.}.

\item Next, following the setup of~\cite{chenchigoldsmith}, we consider the case where the signal coefficients exhibit a \emph{power-law decay}. Specifically, without loss of generality, suppose that the indices of $\xo$ are such that $\abs{x^*_1} \geq \abs{x^*_2} \geq \ldots \abs{x^*_s} > 0$ and ${x_j^*}^2 \leq \frac{C(\alpha)}{j^\alpha}$. Then, we can prove that our CoPRAM algorithm exhibits a sample complexity of $m > \order{s \log n}$, which is very close to the information theoretic limit. 

\item Finally, we consider the case where the underlying signal $\xo$ belongs to an \emph{a priori} specified structured sparsity model. We specifically examine the case of \emph{block-sparse} signals with uniform block size $b$ (i.e., the $s$ non-zeros can be equally grouped into $k = s/b$ non-overlapping blocks). We can equivalently say that the signal $\xo$ has underlying model $\mathcal{M}_s^b$. For this problem, we design a recovery algorithm that we call \emph{Block CoPRAM}. We analyze this algorithm and show that leveraging block-structure further reduces the sample complexity of stable recovery to $\order{ks \log n}$. For sufficiently large block lengths of $b=\omega(s)$, or block sparsity $k\approx 1$, this bound equates to $\order{s \log n}$ which, again, is very close to the information theoretic limit.
We also demonstrate that the more challenging case of \emph{overlapping} blocks can also be solved using our technique, with a constant factor increase in sample complexity.

\end{enumerate}

To our knowledge, this constitutes the first \emph{linearly convergent} series of algorithms for phase retrieval where the (Gaussian) sample complexity has a sub-quadratic dependence on the sparsity level of the signal.  A comparative description of the performance of our algorithms is presented in Table~\ref{tab:compare}.

\begin{table*}[!t]
	\centering
	\caption{{\sl Comparison of our proposed methods with existing approaches for sparse phase retrieval using Gaussian measurements. Here, $n$ denotes signal length, $s$ denotes sparsity, and $k=s/b$ denotes block-sparsity. $\ordereps{\cdot}$ hides polylogarithmic dependence on $\frac{1}{\epsilon}$.}} \label{tab:compare}
	\resizebox{\textwidth}{!}{
		\begin{tabular}{|p{3.3cm}|p{2.9cm}|p{1.9cm}|p{2.3cm}|p{1.3cm}|} 
			\hline
			Algorithm & Sample complexity & Running time & Assumptions & Parameters\\
			%	\hline
			%	AltMin & $\ordereps{n\log^3 n}$ & $\ordereps{n^2 \log^3 n}$ & none & none\\
			\hline 
			AltMinSparse~\cite{netrapalli} & $\ordereps{s^2 \log n + s^2 \log^3 s}$ & $\ordereps{s^2n\log n}$ & $x_{\min}^* \approx {\frsq{c}{s}}\twonorm{\xo}$ & none\\
			\hline 
			$\ell_1$-PhaseLift~\cite{cprl} & $\order{s^2\log n}$ & $\order{\frac{n^3}{{\epsilon^2}}}$ & none & none\\
			\hline
			Thresholded WF~\cite{cai} & $\order{s^2\log n}$ & $\ordereps{n^2 \log n}$ & none & stepsize $\mu$, thresholds $\alpha, \beta$\\
			\hline 
			SPARTA~\cite{sparta} & $\order{s^2\log n}$ & $\ordereps{s^2 n\log n}$ & $x_{\min}^* \approx {\frsq{c}{s}}\twonorm{\xo}$ & stepsize $\mu$, threshold $\gamma$\\
			\hline 
			\hline
			{CoPRAM} (this paper)  & $\order{s^2\log n}$ & $\ordereps{s^2 n\log n}$ & none & none \\
			%(this paper) & & & &\\
			  & $\order{s \log n}$ & $\ordereps{s n\log n}$ & power-law decay & none \\
			\hline
			{Block CoPRAM} (this paper) & $\order{ks \log n}$ & $\ordereps{ks n\log n}$ & none & none\\
			\hline
	\end{tabular}}
\end{table*}
%\subsubsection*{Block sparse signals}

\subsection{Techniques}

\textbf{\textit{Sparse phase retrieval}.} Our proposed CoPRAM algorithm is conceptually very simple. It integrates existing approaches in stable sparse recovery (specifically, the CoSaMP algorithm~\cite{cosamp}) for sparse signal estimation with the alternating minimization approach for phase retrieval proposed in~\cite{netrapalli}\footnote{It is worthwhile to note that the high level idea of alternately estimating the phase and the signal is classical, dating back to the work of Gerchberg and Saxton~\cite{GerchbergS72}.}.

A similar integration of sparse recovery with alternating minimization was also introduced in the work of~\cite{netrapalli}; however, their approach only succeeds when the true support of the underlying {signal} is accurately identified during initialization. This can be fairly challenging to achieve in realistic situations when the signal coefficients are of differing magnitudes. Instead, CoPRAM permits the support of the estimate to evolve across iterations, and therefore can iteratively ``correct" for any errors made during the initialization. Moreover, their analysis requires using fresh samples for every new update of the estimate, while our analysis succeeds in the (more practically useful) setting of using all the available samples at our disposal.

Our first challenge is to identify a good initial guess of the {signal}. As is the case with most non-convex algorithmic techniques, CoPRAM requires an initial estimate $\xin$ that is relatively close to the true vector $\xo$. To this end, we use a variant of the spectral initialization procedure previously proposed in~\cite{sparta}. The basic idea is to identify ``important" co-ordinates by constructing suitable biased estimators of each signal coefficient, followed by a specific eigendecomposition. However, the initialization in CoPRAM is far simpler than the approach in~\cite{sparta}; we perform no pre-processing of the measurements and our method requires no tuning parameters other than the sparsity level $s$.  We also provide a novel analysis of this modified initialization procedure. A drawback of the theoretical results of~\cite{sparta} is that they impose a minimum requirement on every non-zero entry of the true vector $\xo$: $x_{min}^* \equiv \min_{j\in S} |x_j^*| = {C}\twonorm{\xo}/{\sqrt{s}}$. However, this assumption is equivalent to supposing that all nonzero coefficients are approximately the same magnitude, which can be unrealistic; in the case of real-world signal and image data, the (sorted) coefficients usually obey a power-law decay, which violates these constraints. Our analysis removes this requirement; the high level idea in our approach is to show that a coarse estimate of the support will suffice, since any errors in support identification necessarily have to coincide with small coefficients. Our approach also differs from the method adopted in \cite{cai}, which selects indices corresponding to large coefficients based on a parameter-dependent threshold value. The  support estimation step of our algorithm, coupled with the spectral decomposition method in~\cite{cai} gives us a suitable initialization. We prove that the sample complexity for achieving this initial estimate of $\xo$ is $\order{s^2 \log n}$, matching that of the best available previous methods. 

Our next challenge is to show that starting from a good initial guess, an alternating procedure that switches between estimating the phases and estimating the sparse signal (using CoSaMP) converges rapidly to the desired solution. To this end, we unpack the analysis of the CoSaMP algorithm provided in~\cite{cosamp}. In particular, we show that any ``phase errors" made in the initialization step can be suitably controlled across different estimates. As a key step in our analysis, we leverage a recent result by~\cite{mahdi} that shows sufficient decrease in the signal estimation error using the generic chaining technique of~\cite{talagrand,dirksen}. Here too, our algorithm requires no tuning parameters other than the sparsity level $s$. 

\textbf{\textit{Block-sparse phase retrieval}.} We can then use CoPRAM to establish its extension Block CoPRAM, which is a novel phase retrieval strategy for block sparse signals, which have been sampled using generic Gaussian measurements. 
Again, the algorithm is based on a suitable initialization followed by an alternating minimization procedure, and the algorithmic steps exactly mirror those of CoPRAM. To our knowledge, this is the first results for phase retrieval under more refined structured sparsity assumptions on the {signal}. 

As above, the first challenge is to identify a good initial guess of the solution in the first stage. We proceed as in CoPRAM, but instead of identifying important co-ordinates, we instead isolate \emph{blocks} of nonzero coordinates. The high level idea is to construct a different, specially chosen biased estimator for the ``mass" of each block. We prove that a good initialization can be achieved using this procedure using only $\order{ks \log n}$ generic measurements. When the block-size is large enough, the sample complexity of the initialization can be \emph{sub-quadratic} in the sparsity $s$. Specifically, for $b = \Theta(s)$ the sample complexity is only a logarithmic factor away from the information-theoretic limit $\order{s}$.

The second challenge is to demonstrate rapid descent to the desired solution in the second stage. To this end, we replace the CoSaMP sub-routine in CoPRAM with the \emph{model-based CoSaMP} algorithm of \cite{modelcs}, specialized to block-sparse recovery. The analysis proceeds analogously as above. To our knowledge, this constitutes the first end-to-end linearly convergent algorithm for phase retrieval (with generic Gaussian measurements) that demonstrates a sub-quadratic dependence on the sparsity level of the signal.

\subsection{Paper organization}

The remainder of the paper is organized as follows. In Section \ref{subsec:priorwork} we provide a brief overview of prior work. In Section \ref{sec:preliminaries}, we present preliminaries and notation used for our analysis. In Sections \ref{sec:copram} and \ref{sec:bcopram}, we introduce the CoPRAM and Block-CoPRAM algorithms respectively, and provide a theoretical analysis of their statistical as well as computational performance. In Section \ref{sec:results} we provide a series of numerical experiments demonstrating the performance of our algorithms, and in Section \ref{sec:remarks} we provide concluding remarks.

\section{Related work} \label{subsec:priorwork}

\subsection{Prior work}

The phase retrieval problem has received significant attention in the past few years. Attempts to solve this problem have mainly fallen into one of two broad solution approaches: convex and non-convex.

%\subsubsection{Convex approaches}
Convex approaches involve linearizing the problem by {lifting} the signal $\xo$ into a higher-dimensional space and solving a constrained optimization problem. Popular methodologies to solve the problem in the lifted framework include the seminal \emph{PhaseLift} approach and its variations \cite{phaselift,gross2017improved,candes2015phase}; along similar lines is the PhaseCut approach~\cite{phasecut} which proposes a phase retrieval approach based on an SDP relaxation of the MaxCut problem. 
However, most lifting based approaches suffer severely in terms of computational complexity. A recent convex approach that does not use the {lifting} procedure is \emph{PhaseMax}, which produces a novel relaxation of the phase retrieval problem similar to basis pursuit~\cite{goldstein2016phasemax}. While theoretically sound, the empirical performance of PhaseMax is not competitive with other lifting-based approaches.
%\subsubsection{Non-convex approaches}

On the other hand, nonconvex algorithms typically consist of two stages: finding a good initial point, followed by minimizing a loss function via a procedure similar to gradient-descent. The loss function being minimized can be either a function of a {quadratic} form involving the unknown signal, or a function involving the modulus of the inner products with the signal with the measurement vectors. Approaches based on Wirtinger Flow (WF)~\cite{CandesLS14,twf,cai,zhang2016reshaped} popularly use the quadratic form, while approaches based on Amplitude Flow (AF)~\cite{TAF,sparta} as well as stochastic algorithms based on the Kaczmarz method~\cite{wei2015solving} use the modulus form. In \cite{trustregion} Sun \emph{et al.} describe a polynomial-time trust-region algorithm that uses arbitrary initializations to find the global optimum.

%subsubsection{Sparse phase retrieval}
Recent works have adapted the phase retrieval framework for the case when the underlying signal is {sparse}. Some of the convex approaches include \cite{cprl,li2013sparse}, which uses a combination of trace-norm and $\ell$-norm relaxation, to solve an SDP problem. Constrained sensing vectors have been used by Bahmani and Romberg \cite{bahmani} to effectively de-couple the problem into a phase retrieval stage followed by a sparse signal recovery stage, at optimal sample complexity of $\order{s\log\frac{n}{s}}$.  Fourier measurements have been studied extensively in \cite{jaganathan2012recovery} in the convex setting. Similarly, non-convex approaches for sparse phase retrieval include \cite{netrapalli,sparta,cai} which achieve sample complexities of $\order{s^2 \log n}$. Fourier measurements have been evaluated in the non-convex setting in \cite{shechtman2014gespar}, where they use a local search method to solve the sparse phase retrieval problem. Separate from this line of work, Schniter and Rangan in \cite{schniter2015compressive} have proposed an approximate message passing algorithm to experimentally exhibit a sample complexity of $\order{s \log \frac{n}{s}}$.

%\subsubsection{Structured sparsity models}
Going beyond sparsity, a natural approach is to \emph{refine} the assumptions on the nonzero signal coefficients so as to better model various types of real-world phenomena. Structured sparsity models have been proposed to leverage combinatorial interactions such as groups, blocks, clusters, trees, and various other refinements that can be used to model the signal of interest. Applications of structured sparsity models have been developed for sparse recovery~\cite{modelcs,eldar2010block,csmrf,clustersICML,approxSODA,surveyEATCS,MarcoCISS,modelcs,treesISIT,modelcsICALP} as well as in high-dimensional optimization and statistical learning \cite{grouplasso,huang2011learning}. However, to the best of our knowledge, there has been no rigorous results that explore the impact of structured sparsity models for the phase retrieval problem.

\subsection{Subsequent work}

Since the initial appearance of this paper on Arxiv and its subsequent conference publication~\cite{jagatap2017phase}, numerous related works have emerged. Of notable interest has been Waldsperger's new result for phase retrieval which shows that standard alternating minimization provably converges \emph{without} any special initialization, albeit with much higher sample complexity \cite{waldspurger2016phase}. At the moment, we do not know how to design initialization-free algorithms in the case of \emph{sparse} phase retrieval, and leave that as potential future work. Other recent works include re-weighted amplitude flow \cite{wang2017solving} and convolutional phase retrieval \cite{qu2017conv}. Our framework proposed in this paper has also spurred follow-up work in an optical imaging application called {Fourier ptychography} with promising results; see~\cite{sfp,lrfp} for details. 
%Notation
\section{Preliminaries}\label{sec:preliminaries}
%\vspace*{-5pt}

We introduce some notation that will be used throughout the paper. We use bold capital-case letters ($\A,\mat{P}$, etc.) to denote matrices, bold small-case letters ($\x,\y$, etc.) to denote vectors, and non-bold letters ($\alpha, c$ etc.) for scalars. We use $\vect{x}^\top$ and $\A^\top$ to denote the transpose of the vector $\vect{x}$ and the matrix $\A$ respectively. The diagonal matrix form of a column vector $\y \in \reals^m$ is represented as diag$(\y)$, which is the matrix in $\rmm$ with its diagonal elements as $\y$ and all off-diagonal entries are zero. The cardinality of set $S$ is expressed using the operator $\card{S}$.

In this paper we use the standard Gaussian (or normal) distribution over $\rn$ (i.e., the elements of $\vect{a}$ are distributed according to the distribution $\normal$). We define $\sign{x} \equiv \frac{x}{|x|}$ for every $x \in \reals, x \neq 0$, with the convention that $\sign{0} = 0$. Also, we define $\distop{\vect{x}_1}{\vect{x}_2} \equiv \min(\dist{\x_1}{\x_2},\twonorm{\x_1+\x_2})$ for every $\x_1,\x_2 \in \reals^n$. The projection of vector $\x \in \rn$ onto a set of coordinates $S$ is represented as $\x_S \in \rn$, ${x_S}_j = x_j \:\: \text{for} \:\: j\in S, \text{and}\hspace{0.05cm}\:0\:\: \text{ elsewhere}$. Projection of matrix $\M \in \rnn$ onto $S$ is $\M_S \in \rnn$, ${M_S}_{ij} = M_{ij} \:\: \text{for} \:\: i,j\in S, \text{and }\hspace{0.01cm}\:0\:\: \text{elsewhere}$.
For the sake of improved computational complexity, for algorithmic implementations, $\x_S$ can be assumed to be a truncated vector $\x \in \reals^s$,  discarding all elements in $S^c$.
The element-wise inner product of two vectors $\x_1$ and $\x_2 \in \rn$ is represented as $\x_1\circ \x_2 $. $C$ has been used to denote unspecified constants that are \textit{large enough}. Similarly $\delta$ has been used for small constants. The abbreviations \textit{wlog} and \textit{whp} denote ``without loss of generality" and ``with high probability" respectively.

%Algorithm
\section{Compressive phase retrieval}\label{sec:copram}
%\vspace*{-5pt}

In this section, we propose a new algorithm for solving the sparse phase retrieval problem and analyze its performance. Later, we will show how to extend this algorithm to the case of more refined structural assumptions about the underlying sparse signal.

We first provide a brief outline of our proposed algorithm. It is clear from the discussion in the introduction that the recovery problem \eqref{eqn:magnitude-measurements} is highly non-convex, and multiple locally optimal solutions might exist. Therefore, as is typical in modern non-convex methods~\cite{netrapalli,sparta,optspace} we use an spectral technique to obtain a good initial estimate. This technique itself is a modification of the initialization stage of Algorithm 1 of \cite{sparta}; however, as we discuss below, our method requires no special tuning parameters except for knowledge of the underlying sparsity $s$. Moreover, in contrast with \cite{sparta} our theoretical analysis requires no extra assumptions on the signal coefficients. 

Once an appropriate initial estimate is chosen, we then show that a simple alternating-minimization algorithm will converge rapidly to the underlying true signal. Our proposed algorithm is new, and builds upon the original alternating-minimization algorithm proposed in~\cite{netrapalli}. In a departure from existing sparse phase retrieval methods~\cite{sparta,cai}, our method is \emph{parameter free} except for knowledge of the sparsity level $s$.

%Initialization for main algorithm
We call our overall algorithm \emph{Compressive Phase Retrieval with Alternating Minimization} (CoPRAM). As described above, the algorithm is divided into two stages: an \emph{Initialization} stage and a \emph{Descent} stage. The stages are presented in pseudocode form as Algorithms \ref{algo:initial_copram} and \ref{algo:copram}.

\begin{algorithm}[!t]
	\caption{CoPRAM: Initialization. 
		%\red{CH: can we turn off numbering of the lines?} 
	}
	\label{algo:initial_copram}
	\begin{algorithmic}[]
		\INPUT $\A,\y,s$
		\STATE Compute signal power: 
		%\begin{align} \label{eq:power}
		$\phi^2 = \frac{1}{m} \sum_{i=1}^{m} y_i^2.$ %\red{Why need this step?}
		%\end{align}
		%	\STATE Compute signal marginals:
		
		%\FOR{$j = 1,\cdots,n$}
		\STATE	Compute signal marginals: $M_{jj} = \frac{1}{m}\sum_{i=1}^m y_i^2 a_{ij}^2$ \quad  $\forall j$. %\red{can we simplify notation to $M_j$?}
		%\ENDFOR
		\STATE Set $\hat{S} \leftarrow j$'s corresponding to top-$s$ $M_{jj}$'s.  
		\STATE Set $\vect{v}_1 \leftarrow \mbox{top singular vector of}~\M_{\hat{S}} = \frac{1}{m}\sum_{i=1}^{m} y_i^2 \aiShat \aiShat^\top \quad \in \reals^{s\times s}$. %, where $\aiShat = \trunc{\ai,\hat{S}}$.
		%\STATE Compute $\vect{v} \in \rn \leftarrow \vect{v}_1$ {for} $\hat{S}$, {and} $\mathbf{0}$ {for} $\hat{S}^c$.
		\STATE  Compute $\vect{x^0} \leftarrow \phi \vect{v} $, where $\vect{v} \leftarrow \vect{v}_1$ {for} $\hat{S}$ {and} $\mathbf{0} \in \reals^{n-s}$ {for} $\hat{S}^c$.
		\OUTPUT $\xin$.
	\end{algorithmic}
\end{algorithm}

%Main algorithm
\begin{algorithm}[!t]
	\caption{CoPRAM: Descent.}
	\label{algo:copram}
	\begin{algorithmic}[]
		\INPUT $\A,\y,\xin,s,t_0$
		\STATE Initialize $\vect{x^0}$ according to Algorithm \ref{algo:initial_copram}.
		\FOR{$t = 0,\cdots,t_0-1$}
		\STATE	$\P^{t+1} \leftarrow \mbox{diag}\left(\sign{\A \vect{x^{t}}}\right)$,
		\STATE	$\vect{x^{t+1}}$ = \textsc{CoSaMP}($\frac{1}{\sqrt{m}}\A$,$\frac{1}{\sqrt{m}}\P^{t+1}\y$,$s$,$\xt$).
		\ENDFOR
		\OUTPUT $\z \leftarrow \vect{x}^{t_0}$.
	\end{algorithmic}
\end{algorithm}

\subsection{Initialization}
%We now study the CoPRAM algorithm, and analyze both its statistical and computational performance.

The first stage of CoPRAM uses a similar approach as those provided in previous sparse phase retrieval methods. The high level idea is to use the measurements $y_i$ to construct a \emph{biased} estimator of the (squared) absolute values of the true signal coefficients. For the $j^\textrm{th}$ signal coefficient, the \emph{marginal} $M_{jj}$ is given by:
$$
M_{jj} = \frac{1}{m} \sum_{i=1}^m {y_i^2 a_{ij}^2},
$$
and the set of all $M_{jj}$'s can be calculated in $\order{mn}$ time. The marginals themselves do not directly produce the signal coefficients, but the ``weight" of each marginal (with sufficiently many samples) can enable identification of the coordinates of the true signal support. Once the support is accurately identified, a spectral technique (e.g., the methods of \cite{netrapalli,sparta,cai}) can be used to construct a good initial estimate $\xin$. 

However, accurate support identification can be tricky in general, particularly in the presence of very small signal coefficients. Indeed, to avoid this issue, earlier works \cite{netrapalli,sparta} assume that the magnitudes of the nonzero signal coefficients are all sufficiently large, i.e., $\Omega\rbrak{\twonorm{\xo}/\sqrt{s}}$. As discussed earlier, this assumption can be unrealistic, violating the power-decay law. 

Our analysis resolves this issue by \emph{relaxing} the requirement of accurately identifying the support. The basic intuition is that even a coarse estimate of the support suffices to achieve a good estimate, since the errors are all going to correspond to small coefficients anyway. Such ``noise" in the signal estimate can be controlled with a sufficient number of samples. A similar argument has been made in the analysis of the initialization stage of \cite{cai}; however, their estimate is a strict subset of the true support and their method requires tuning of real-valued parameters that can be hard to estimate in practice. Instead, we use a simpler estimation procedure. Indeed, we show that a simple pruning step that rejects the  smallest $(n-k)$ coordinates, followed by the spectral procedure of \cite{sparta}, gives us the initialization that we need.

Concretely, we leverage the following fact: if elements of $\A$ are distributed as per standard normal distribution $\mathcal{N}(0,1)$, a weighted correlation matrix $\M$ can be constructed with diagonal elements $M_{jj}$, 
\begin{align} %\label{eq:marginal_mat}
\M&=\frac{1}{m}\sum_{i=1}^{m} y_i^2 \ai \ai^\top, \nonumber\\
M_{jj}&=\frac{1}{m}\sum_{i=1}^{m} y_i^2 a_{ij}^2. \label{eq:marginal_element}
\end{align}
Then the expectation of this matrix $\M$ is,
\begin{align} \label{eq:spectral}
\expec{\M} &= \expec{\frac{1}{m}\sum_{i=1}^{m} y_i^2 \ai\ai^\top}, \\ \nonumber
&= \left(\I_{n\times n} + 2\frac{\xo}{\|\xo\|_2}\cdot\frac{\xo^\top}{\|\xo\|_2}\right)\|\xo\|_2^2,
\end{align}  
where $\M,\expec{\M} \in \rnn$. The diagonal elements of this expectation matrix $\expec{\M}$ are given by:
\begin{align} \label{eq:def_marginals}
\expec{M_{jj}} = 
\begin{cases}
& \|\xo\|^2 + 2x_j^{*2} \quad \text{for}\quad j \in S,\\
& \|\xo\|^2  \hspace{1.25cm}\:\:\:\text{for}\quad j \in S^c.
\end{cases}
\end{align}

Intuitively, the signal marginals at locations on the diagonal of $\M$ corresponding to $j\in S$ are larger, on an average, than those corresponding to the zero-locations $(j\in S^c)$.  Using this as the baseline, we evaluate the diagonal elements of the matrix $\M$ (which we refer to as \emph{marginals}) and establish a threshold value $\Theta$ which separates the indices $j$ corresponding to these marginals into sets $S$ and $S^c$.

%\red{The next step involves making the following observation:
%\begin{align}
%y_i = \iprod{\ai}{\xo} = {\sum_{j=1}^{n} x_j \cdot a_{ij}} = {\sum_{j\in S} x_j \cdot a_{ij}} = \iprod{\ai_S}{\xS},
%\end{align}
%for $i=1,2\dots m$. The observation set $\y$ is independent of the elements of $\ai$ corresponding to $S^c$. Hence the marginals corresponding to $j \in S$ and $j \in S^c$, have to be evaluated separately, which will lead to the required threshold that separates them.
%}

We formalize the above argument. Our first result shows that given a sufficient number of measurements of the form \eqref{eq:samplecomplexity}, this method produces an estimate that is close enough to the true underlying signal.

\begin{theorem}
	\label{thm:main_initial}
	The output of Algorithm~\ref{algo:initial_copram}, $\xin \in \mathcal{M}_s$ is a small constant distance $0 < \delta_0 < 1$ away from the true signal $\xo \in \mathcal{M}_s$, i.e., 
	\begin{align*}
		\distop{\xin}{\xo} &\leq 
		\delta_0 \twonorm{\xo},
	\end{align*}
as long as the number of (Gaussian) measurements `m' satisfies the following bound, 
\begin{align} \label{eq:samplecomplexity}
m \geq C s^2 \log mn,
\end{align}
with probability greater than $1-\frac{8}{m}$.
\end{theorem}

This theorem is proved via Lemmas~\ref{lem:smallmarginals} through \ref{lem:xin}, and the argument proceeds as follows. We evaluate the marginals of the signal $M_{jj}$, in broadly two cases: $j\in S$ and $j\in S^c$. 
%\red{CH: move all proofs to appendix?}
The key idea is to establish one of the following:
\begin{enumerate}
	\item If there is a restriction on the minimum element of the true $\xo$ (i.e., if it is bounded away from zero by a specific amount), then there exists a clear separation between the marginals $M_{jj}$ for $j\in S$ and $j \in S^c$, with high probability. Then one would, whp, pick up the correct support in Algorithm~\ref{algo:initial_copram} (i.e. $\hat{S} = S$). The top-singular vector of the truncated matrix $\M_{S}$ gives a good initial estimate $\xin$.
	\item If there is no such restriction, even then the support picked up in Algorithm~\ref{algo:initial_copram}, $\hat{S}$, contains a bulk of the correct support $S$. Some fraction of the elements picked up in $\hat{S}$ are incorrect, but we prove that they induce negligible error in estimating the initial vector $\xin$. 
\end{enumerate}

These approaches are illustrated in Figures~\ref{fig:illus1} and \ref{fig:illus2} in Appendix \ref{sec:appendixA}. The marginals $M_{jj}$ for $j\in S^c$ are upper bounded as stated in Lemma \ref{lem:smallmarginals}. Similarly, the marginals $M_{jj}$ for $j\in S$ are lower bounded as stated in Lemma \ref{lem:largemarginals}. The identification of the support $\hat{S}$ (which provably contains a significant chunk of the true support $S$) serves as a basis to construct the truncated correlation matrix $\M_{\hat{S}}$. The top singular vector of this matrix $\M_{\hat{S}}$, $\xin$ gives us a good initial estimate of the true signal $\xo$.

%\textbf{Scaling the normalized top-singular vector:}\\

The final step of Algorithm~\ref{algo:initial_copram} requires a scaling by a factor $\phi$. This ensures that the power of the initial estimate $\xin$ is close to the power of the true signal $\xo$ (this is required because the calculation of the top-singular vector gives us a normalized vector $\vect{v}$). Provided sufficiently many samples, the signal power $\twonorm{\xo}^2$ is well approximated by the average power in the measurements $\phi^2$ which is defined as
\begin{align}
\label{eq:power}
\phi^2 = \frac{1}{m} \sum_{i=1}^m y_i^2.
\end{align}
Using Lemma~\ref{lem:signalpower} in Appendix~\ref{sec:appendixC} we can show that
\begin{align*}
%\label{eq:signalpower}
1 - \delta \leq \frac{\phi^2}{\twonorm{\xo}^2} \leq 1 + \delta,
\end{align*}
for small constant $ 0 < \delta \ll 1$ with probability greater than $1-\frac{1}{m}$.\\

\subsection{Descent to optimal solution}

Once we obtain a good enough initial estimate $\xin$ such that $\distop{\xin}{\xo} \leq \delta_0 \twonorm{\xo}$, whp, we now construct a method to accurately recovery the true $\xo \in \mathcal{M}_s$. To achieve this, we adapt the alternating minimization approach from \cite{netrapalli}. %Alternating minimization essentially decouples the phase estimation and the signal estimation problems, and we exploit this modularity o

We introduce some notation. The observation model in \eqref{eqn:magnitude-measurements} can be restated as follows:
\begin{align*}
\sign{\iprod{\ai}{\xo}}\circ y_i = \iprod{\ai}{\xo} ,
\end{align*}
for all $i=\{1,2,\dots m\}$.
We denote the \emph{phase vector} $\vect{p} \in \reals^m$ as a vector that contains the (unknown) signs of the measurements, i.e., ${p}_i = \sign{\iprod{\ai}{\x}}$ for all $i=\{1,2,\dots m\}$. We can also define a diagonal \emph{phase matrix} $\P \in \rmm$, such that $\P = \text{diag}\rbrak{\vect{p}}$. Then our measurement model gets modified as:
\begin{align*}
\P^*\y = \A \xo,
\end{align*}
where $\P^*$ denotes the true phase matrix. Consider minimizing the loss function composed of two variables $\x$ and $\P$,
\begin{align} \label{eq:lossfunc}
\min_{\zeronorm{\x} \leq s,\P \in \mathcal{P}} \twonorm{\A\x - \P\y}.
\end{align}
%Since we do not know $\P^*$, one approach to recovering $\xo$ is to solve:
%\begin{equation} \label{eq:ls}
%\argmin_{\P,\x} ~ \|\A \x - \P\y\|_2, 
%\end{equation}
%where $\x \in \rn$ and $\P \in \rmm$ is a diagonal matrix with each diagonal entry being equal to $\pm1$. 
Note that the problem above is {\em not convex}, because $\P$ is restricted to be a diagonal matrix $\in \mathcal{P}$, where $\mathcal{P}$ is a set of all diagonal matrices with diagonal entries constrained to be in $\{-1,1\}$. Instead, we alternate between estimating $\P$ and $\x$.
%Netrapalli, et. al., in \cite{netrapalli}, use the well-known alternating minimization technique to obtain the solution of this problem: one updates $\x$ and $\P$ in an alternating way, so as to minimize \eqref{eq:ls}. 
We perform two estimation steps:
\begin{enumerate}
	\item 
 If we fix the signal estimate $\x$, then the minimizer $\P \in \mathcal{P}$ is given in closed form as:
 \begin{align} \label{eq:phase_est}
 \P=\mbox{diag}\left(\sign{\A\x}\right) .
 \end{align} 
 We call this the \textit{phase estimation} step.
	\item  If we fix the phase matrix $\P \in \mathcal{P}$, the signal vector $\x$ can be obtained by solving a sparse recovery problem,
\begin{align} \label{eq:loss_min}
 \min_{\x, \zeronorm{\x} \leq s} \|\A \x-\P\y\|_2 ,
\end{align}
 if $m < n$ (and each entry of $\A$, $a_{ij}$ is sampled from independent Gaussian $ \normal$, such that $\frac{\A}{\sqrt{m}}$ satisfies the \textit{restricted isometry property}). We call this the \textit{signal estimation} step.
 
\end{enumerate}

We employ the CoSaMP \cite{cosamp} algorithm to (approximately) solve \eqref{eq:loss_min}. Note that since \eqref{eq:loss_min} itself is a non-convex problem and exact minimization can be hard. However, we do not need to explicitly obtain the minimizer but only show a sufficient descent criterion, which we achieve by performing a careful analysis of the CoSaMP algorithm. For analysis reasons, we require that the entries of the input sensing matrix are distributed according to $\gauss\rbrak{0,{\I}/{\sqrt{m}}}$. This can be achieved by scaling down the inputs to CoSaMP: $\A^t,\P^{t+1}\y$ by a factor of $\sqrt{m}$ (see $\x$-update step of Algorithm~\ref{algo:copram}). Another distinction is that we use a ``warm start" CoSaMP routine for each iteration where the initial guess of the solution to~\eqref{eq:loss_min} is given by the current signal estimate. 

We now analyze our proposed descent scheme. We obtain the following theoretical result:
\begin{restatable}{theorem}{convergence}\label{thm:convergence}
	Given an initialization $\xin \in \mathcal{M}_s$ satisfying $\distop{\x^0}{\xo} \leq \delta_0 \twonorm{\xo}$,	for $0 < \delta_0 < 1$, if we have number of (Gaussian) measurements $m > Cs \log \frac{n}{s} $, then the iterates of Algorithm~\ref{algo:copram} satisfy:
	\begin{align} \label{eq:mainconvergence}
	\distop{\xtplus}{\xo} \leq {\rho_0} \distop{\xt}{\xo},
	\end{align}
	where $ 0 < \rho_0 < 1$ is a constant, with probability greater than $1-e^{-\gamma m}$, for positive constant $\gamma$.
\end{restatable}
The proof of this theorem can be found in Appendix~\ref{sec:appendixB}. 

Combining both stages, the number of measurements are required to obey the following lower bound:
\begin{align} \label{eq:m_require}
m_0 > \max\rbrak{C_1s^2\log mn, C_2 s\log\frac{n}{s}} \equiv C s^2 \log mn ,
\end{align}
for the overall CoPRAM algorithm to succeed.

%%%%%%%%%%%%%%%%%%%%%%%%%%%%%%%%%%%%%%%%%%%%%%%%%%%%%%%%%%%%%%%%%%%%%%%%%%%%%%%%%%%%%%%%%%

\subsection{Robustness to noise}

The above analysis assumes that the measurements are pristine (noiseless). We can also demonstrate that the CoPRAM algorithm are sufficiently robust in the presence of noise.  This is established in the following theorem.
\begin{restatable}{theorem}{noise} \label{thm:noise}
	Given Gaussian measurements $a_{ij} \in \normal$, CoPRAM can recover the model sparse signal $\x^{t_0}\in \mathcal{M}_s$ from noisy measurements $\y$ of the form
	\begin{align*}
		\y = \abs{\A \xo} +{\epsilon},
	\end{align*} 
	where $\epsilon\in \reals^m$ is a scaled sub-exponential. 
	This retains the previously derived expression for sample complexity as in Theorem \ref{thm:main_initial} up to a constant factor. The algorithm converges according to the iteration invariant:
	\begin{align*}
	\twonorm{\x^{t_o}-\xo} \leq c_1\twonorm{\xo} + c_2 \twonorm{\epsilon}
	\end{align*}
	where $t_o$ is the number of outer iterations of CoPRAM and Block CoPRAM, $c_1 < 1$ and $c_2 = 0.5$.
\end{restatable}
The proof for this theorem can be found in Appendix \ref{sec:AppendixE}. An identical analysis holds for the block-sparse case which we elaborate in more detail below.

\subsection{Sparse signals exhibiting power law decay}

The quadratic dependence of the sample complexity on the signal sparsity level $s$, as derived in Theorem \ref{thm:main_initial}, is typical (and also shared by the other works~\cite{twf,sparta}) but somewhat problematic. In particular, if the signal sparsity exceeds the square-root of the dimension $n$, the result becomes moot and one may as well as standard phase retrieval techniques!

In this section, we demonstrate a method to break through this quadratic barrier, albeit under somewhat more stringent assumptions on the signal. Specifically, we analyze the scenario hypothesized in~\cite{chenchigoldsmith} in which the signal $\xo$ follows a power-law decay. That is, suppose that the signal coefficients have been suitably re-indexed such that ${x_j^*}^2$, for $j\in 1,2,\ldots,s$ can be arranged in non-increasing order: 
\begin{align} \label{eq:PLdecay}
{x_j^*}^2 \leq  \frac{C(\alpha)}{j^\alpha} 
\end{align}
where  and $\alpha >1$. Due to the isotropic nature of the Gaussian measurement scheme, such a re-indexing can be assumed without loss of generality. 

We can show that this extra power-law decay assumption results in \emph{far fewer} samples, $\order{s\log n}$ for the CoPRAM initialization step to achieve a sufficiently good initial guess. Combined with Theorem~\ref{thm:convergence}, we obtain the following result:

\begin{restatable}{theorem}{PLdecay} \label{thm:PLdecay}
		Given Gaussian measurements $a_{ij} \in \normal$, then CoPRAM can recover the $s$-sparse signal $\x^{t_0} \in \mathcal{M}_s$, with $\twonorm{\x^{t_0} - \xo} \leq \delta \twonorm{\xo}$, where $t_0$ is the number of outer iterations of CoPRAM, from $m > Cs\log n$ measurements, as long as the coefficients of the signal follow a power-law decay as described in \eqref{eq:PLdecay}. 
\end{restatable}

The proof for this theorem can be found in Appendix \ref{sec:AppendixPL}.

% the effective number of uniform non-zero blocks gets doubled as $k' \leq \frac{2s}{b}$. Equivalently, the size of uniform block gets halved $b' \geq \frac{b}{2}$. The sample complexity retains the same order $\order{ks \log n}$.

%%%%%%%%%%%%%%%%%%%%%%%%%%%%%%%%%%%%%%%%%%%%%%%%%%%%%%%%%%%%%%%%%%%%%%%%%%%%%%%%%%%%%%%%%%

\section{Block-sparse phase retrieval} \label{sec:bcopram}

The analysis of the proofs mentioned so far, as well as experimental results suggest that we can reduce sample complexity for successful sparse phase retrieval by exploiting further structural information about the signal. We assume that the true signal $\xo$, is block sparse with uniform block length $b$ and effective block sparsity $k=\frac{s}{b}$. We introduce the concept of \textit{block marginals}, a block-analogue to signal marginals, which can be analyzed to crudely estimate the block support of the signal in consideration. We use this formulation, along with the alternating minimization approach that uses model-based CoSaMP \cite{modelcs} to descend to the optimal solution. In the next subsections, we discuss the initialization and descent of the Block CoPRAM algorithm. The pseudo-code for Block CoPRAM is stated in Algorithms \ref{algo:initial_blk} and \ref{algo:altminblock}.

%Initialization for main algorithm - block
\begin{algorithm}[t]
	\caption{Block CoPRAM: Initialization.}
	\label{algo:initial_blk}
	\begin{algorithmic}[]
		\INPUT $\A,\y,b,k.$
		\STATE Compute signal power $\phi^2 = \frac{1}{m} \sum_{i=1}^{m} y_i^2.$ %\red{why need this step?}
		\STATE Compute block marginals $M_{{j_b}{j_b}} = \sqrt{\sum_{j\in j_b} M_{jj}^2}\quad\quad \forall j_b$, where $M_{jj}$ is as in \eqref{eq:marginal_element}.
		
		\STATE Select $\hat{S}_b \leftarrow j_b$'s corresponding to top-$k$ $M_{{j_b}{j_b}}$'s, $\hat{S}$ is signal support corresponding to blocks $\hat{S}_b$.  
		\STATE Compute $\vect{v}_1 \leftarrow \mbox{top singular vector of } \M_{\hat{S_{b}}} = \frac{1}{m}\sum_{i=1}^{m} y_i^2 \aiShat \aiShat^\top \quad \in \reals^{s\times s}$. %,  where $\aiShat = \trunc{\ai,\hat{S}}$.
		\STATE  Compute $\vect{x^0} \leftarrow \phi \vect{v} $ where $\vect{v} \leftarrow \vect{v}_1$ {for} $\hat{S}$, {and} $\mathbf{0} \in \reals^{n-s}$ {for} $\hat{S}^c$.
		\OUTPUT $\xin$.
	\end{algorithmic}
\end{algorithm}

%Main algorithm -  block
\begin{algorithm}[t]
	\caption{Block CoPRAM: Descent.}
	\label{algo:altminblock}
	\begin{algorithmic}[]
		\INPUT $\A,\y,\xin,b,k,t_0$
		\STATE Initialize $\vect{x^0}$ according to Algorithm \ref{algo:initial_blk}.
		\FOR{$t = 0,\cdots,t_0-1$}
		\STATE	$\P^{t+1} \leftarrow \mbox{diag}\left(\sign{\A \vect{x^{t}}}\right)$.
		\STATE	$\xtplus \leftarrow \argmin_{\x\in \rn} \twonorm{\A \x - \P^{t+1}\y}$ = BlockCoSaMP($\frac{1}{\sqrt{m}}\A$,$\frac{1}{\sqrt{m}}\P^{t+1}\y$,$b$,$k$,$\x^t$).
		\ENDFOR
		\OUTPUT $\z \leftarrow \vect{x}^{t_0}$.
	\end{algorithmic}
\end{algorithm}

\subsection{Initialization}
Block-sparse signals $\xo$, can be said to be following a sparsity model $\mathcal{M}_{s}^b$, where $\mathcal{M}_{s}^b$  describes the set of all block-sparse signals with $s$ non-zeros being grouped into uniform pre-determined blocks of size $b$, such that block-sparsity $k=\frac{s}{b}$. The effective sparsity of the signal is still $s$, however the non-zero elements are constrained to appear in blocks. We use the index set $j_b = \{1,2\dots k\}$, to denote block-indices.

Analogous to the concept of marginals defined above, we introduce \emph{block marginals} $M_{j_b j_b}$, where $M_{jj}$ is defined as in \eqref{eq:marginal_element}. For block index $j_b$, we define:
\begin{align} \label{eq:marginal_block}
M_{{j_b}{j_b}} &= \sqrt{\sum_{j\in j_b} M_{jj}^2}, %\quad \quad \hspace{1cm} j_b \in S_b
\end{align}
to develop the initialization stage of our \textit{Block CoPRAM} algorithm. Similar to the proof approach of CoPRAM, we show that there exists a threshold that separates the block marginals $\M_{j_b j_b}$, for $j_b \in S_b$ and $\M_{j_b j_b}$, for $j_b \in S_b^c$ respectively. Here, $S_b$ represents the ``block support", i.e., the set of active block-indices. We can then evaluate the block marginals, and use the top-$k$ such marginals to obtain a crude approximation $\hat{S_b}$ of the true block support $S_b$. This support can be used to construct the truncated correlation matrix $\M_{\hat{S_b}}$. The top singular vector of this matrix $\M_{\hat{S_b}}$ gives a good initial estimate $\xin$ for the Block CoPRAM algorithm (Algorithm \ref{algo:altminblock}). Through the evaluation of block marginals, we proceed to prove that the sample complexity required for a good initial estimate (and subsequently, successful signal recovery of block sparse signals) is given by $\order{(s^2/b) \log n} = \order{ks\log n}$. This essentially reduces the sample complexity of signal recovery by a factor equal to the block-length $b$ over the sample complexity required for standard sparse phase retrieval. 

Formally, we obtain the following result:
\begin{restatable}{theorem}{blockinit}
	\label{thm:block_initial}
	The initial vector $\xin \in \mathcal{M}_s^b$, which is the output from Algorithm~\ref{algo:initial_blk}, is a small constant distance away from the true signal $\xo \in \mathcal{M}_s^b$, i.e.
	\begin{align*}
	\distop{\xin}{\xo} &\leq \delta_b \twonorm{\xo},
	\end{align*}
	for $0 < \delta_b < 1$, as long as the number of measurements satisfy
	\begin{align*}
	 m \geq C \frac{s^2}{b} \log mn,
	\end{align*}
	with probability greater than $1-\frac{8}{m}$.
\end{restatable}

The proof can be found in Appendix~\ref{sec:appendixA2}, and carries forward intuitively from the proof of the sparse phase-retrieval framework.
 
\subsection{Descent to optimal solution}

For the descent of Block CoPRAM to optimal solution, the phase-estimation step is the same as that in CoPRAM \eqref{eq:phase_est}. For the signal estimation step, we attempt to solve the same minimization as in \eqref{eq:loss_min}, except with the additional constraint that the signal $\xo$ is \textit{block sparse},
\begin{align} \label{eq:loss_min_model}
\min_{\x \in \mathcal{M}_{s}^b} \|\A \x-\P\y\|_2 ,
\end{align}
where $\mathcal{M}_{s}^b$ describes the block sparsity model. In order to approximate the solution to \eqref{eq:loss_min_model}, we use the \emph{model-based CoSaMP} approach of \cite{modelcs}. This is a straightforward specialization of the CoSaMP algorithm and has been shown to achieve improved sample complexity over existing approaches for standard sparse recovery.

Similar to Theorem~\ref{thm:convergence} above, we obtain the following result:
\begin{restatable}{theorem}{blockconvergence}\label{thm:blockconvergence}
	Given an initialization $\xin \in \mathcal{M}_s^b$, satisfying $	\distop{\x^t}{\xo} \leq \delta_b\twonorm{\xo}$, where $0 < \delta_b < 1$, if we have number of measurements $m \geq C\rbrak{s+\frac{sb}{n} \log \frac{n}{s}}$, then the iterates of Algorithm~\ref{algo:altminblock} satisfy:
	\begin{align} \label{eq:blockmainconvergence}
	\distop{\xtplus}{\xo} \leq {\rho_b} \distop{\xt}{\xo}.
	\end{align}
	where $0 < \rho_b < 1$ is a constant, with probability greater than $1-e^{-\gamma m}$, for positive constant $\gamma$.
\end{restatable}
The proof of this theorem can be found in Appendix~\ref{sec:appendixB}.

\subsection{Extension to blocks of non-uniform sizes}

The analysis so far has been made for uniform blocks of size $b$. However the same algorithm can be extended to the case of sparse signals with \emph{non-uniform} blocks. Such a model is particularly useful for time-series signals where the nonzeros occur in ``bursts" of variable lengths and start times.

Formally, consider the \emph{clustered sparsity} model for 1D signals in $\rn$, comprising signals with $s$ non-zeros that occur in no more than $k$ non-overlapping blocks (clusters), each of which exhibit potentially unknown sizes and locations. The above analysis does not immediately apply to this case; however, by the analysis approach of~\cite{ModelCSSAMPTA}, we can show that any such clustered-sparse signal with parameters $(s,k)$ can be \emph{simulated} using a \emph{uniform} block-sparse signal with parameters $(s,3k)$. Therefore, the only price to be paid is a tripling of the block sparsity parameter $k$. Provided we are willing to tolerate this increase, we can use exactly the same Block CoPRAM algorithm (including both the initialization as well as the descent stages) as described above, with only a constant factor increase in the sample complexity. 

We note that this argument is only applicable to block-sparse 1D signals (such as time-domain signals); extending this argument to general clustered-sparse images and higher-dimensional data is much more involved, and we leave this to future work.

%Experiments
%Experiments
\section{Numerical experiments} \label{sec:results}
In this section, we present the results of a range of simulations supporting our algorithms and demonstrate their benefits over the state-of-the-art in sparse phase retrieval. All numerical experiments were conducted using MATLAB 2016a on a desktop computer with an Intel Xeon CPU at 3.3GHz and 8GB RAM.

%\subsection{Synthetic data}
%For the first experiment, we fix the parameters as stated in Table~\ref{tab:params}.
%
%\begin{table}[!t]
%	\centering
%	\caption{List of parameters for the experiment.} \label{tab:params}
%	\begin{tabular}{|c|c|c|c|c|}
%		\hline 
%		sparsity $s$ & block length $b$ & signal length $n$ & no. of samples $m$ & epochs $T$\\
%		\hline
%		 $5:5:50$ & $5$ &  $3000$ & $200:200:2000$  & $50$\\
%		\hline 
%	\end{tabular}
%\end{table}
Our experiments explores the performance of the CoPRAM and Block CoPRAM algorithms on synthetic data.
The non-zero elements of the test signal $\xo \in \rn$, with $n=3,000$ are generated using zero-mean Gaussian distribution $\normal$ and normalized, such that $\norm{\xo} = 1$. The elements of sensing matrix $\A \in \rmn$, $a_{ij}$ are also generated using the zero-mean Gaussian  distribution $\normal$. The sparsity levels $s$ are chosen in steps of $5$ with a maximum value of $s=50$ such that $n=3000 \gtrsim 50^2$ (for values of $s > \sqrt{n}$, the effect of sparsity is minimal and standard non-sparsity based phase retrieval algorithms perform equally well). A block length of $b=5$ is considered for all generated signals in experiments in Figures \ref{fig:phase_graph1}, \ref{fig:phase_trans} and \ref{fig:noise_err}. The number of measurements $m$ is swept from $m = 200$ to $m = 2,000$ in steps of $200$. We repeated each of the experiments (fixed $n,s,m$) in Figures \ref{fig:phase_graph1}, \ref{fig:phase_trans} and \ref{fig:noise_err} for $50$ independent Monte Carlo trials, and the experiments in Figure \ref{fig:phase_trans_blk} for $200$ independent Monte Carlo trials.

For our simulations, we compared our algorithms CoPRAM and Block CoPRAM with two other sparse phase retrieval algorithms: Thresholded Wirtinger Flow~\cite{cai} and SPARTA \cite{sparta}. For our set of generated signals, the AltMinSparse method mentioned in \cite{netrapalli} does not recover the signal in most cases (if the initialization stage fails to pick the correct support, the subsequent AltMinPhase procedure can never give a good solution). We therefore do not include this algorithm for comparisons. 

For Thresholded WF, we set parameters which were optimized based on a number of trial cases and were kept constant throughout all experiments, with values $\alpha = 1.5$, $\mu = 0.23$ and $\beta = 0.3$. Similarly, for SPARTA, we set the parameters to be $\gamma = 0.7$, $\mu = 1$ and $\card{\mathcal{I_o}} = \lceil {\frac{m}{6}}\rceil$ as mentioned in their paper. 

For the first experiment, we generated phase transition plots by evaluating the probability of successful recovery, i.e. number of trials out of $50$, that gave a relative error in retrieval $\frac{\dist{\x^{t_0}}{\xo}}{\twonorm{\xo}} < 0.05$. We let each of the algorithms to run for a total of $t_0 = 30$ iterations. The recovery probability for varying values of $s$ and $m$ has been illustrated in Figure~\ref{fig:phase_trans} through phase transition diagrams. It can be noted that CoPRAM (\ref{fig:phase_graph1} (c)) and SPARTA (\ref{fig:phase_graph1} (b)) perform comparably, while Block CoPRAM (\ref{fig:phase_graph1} (d)) performs the best among all four algorithms, in terms of sample complexity. 
\begin{figure}[!t]
	\centering
	\subfloat[][Thresholded WF]{\includegraphics[width = 0.24\textwidth]{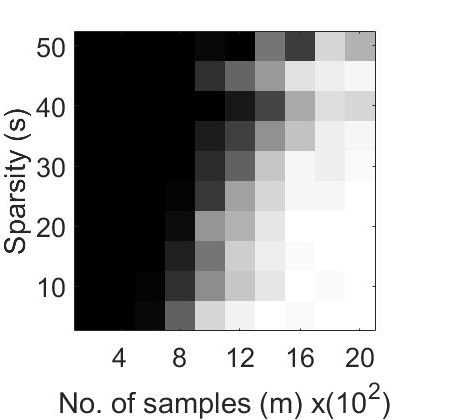}} 
	\subfloat[][SPARTA]{\includegraphics[width = 0.24\textwidth]{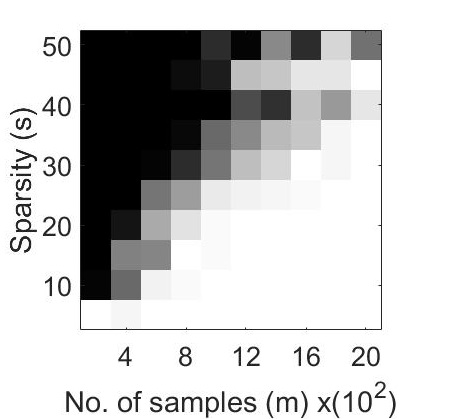}}
	\subfloat[][CoPRAM]{\includegraphics[width = 0.24\textwidth]{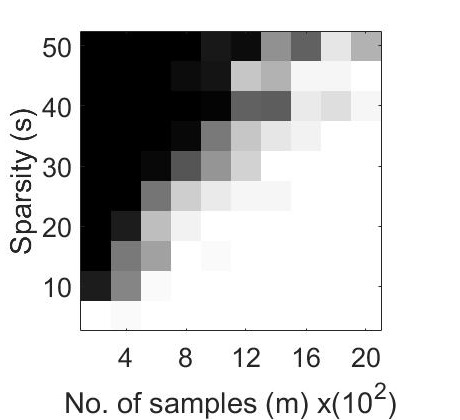}}
	\subfloat[][Block CoPRAM]{\includegraphics[width = 0.24\textwidth]{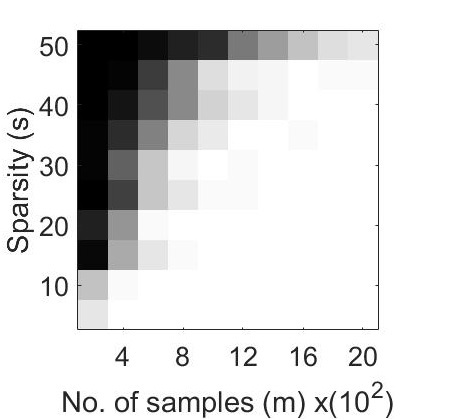}}
	\caption{\sl Phase transition plots for different algorithms, with signal length $n=3000$, having uniform block length of $b=5$.} \label{fig:phase_trans}
\end{figure}
The phase transition graphs for $s=20$ and $s=30$ for the four algorithms is displayed in Figure~\ref{fig:phase_graph1}.
\begin{figure}[!t]
	\centering
	\subfloat[][$s=20$]{
	\begin{tikzpicture}
\begin{axis}
[width=0.4\textwidth,
xlabel= Number of samples $m$, 
ylabel= Probability of recovery,
grid style = dashed,
grid=both,
legend style=
{
at={(1.55,0.7), %for arxiv version
%at={(0.73,1.7),  %for NIPS version
anchor=north east},
cells={align=center}, 
} 
]

\addplot[color=blue,mark=square*] plot coordinates {
	(200,     0)
	(400,    0.12)
	(600,    0.74)
	(800,    0.94)
	(1000,   1)
	(1200,   1)
	(1400,   1)
	(1600,   1)
	(1800,   1)
	(2000,   1)
};
\addlegendentry{CoPRAM};

\addplot plot coordinates {
(200,    0.14)
(400,    0.58)
(600,    0.98)
(800,    1)
(1000,   1)
(1200,   1)
(1400,   1)
(1600,   1)
(1800,   1)
(2000,   1)
};

\addlegendentry{Block \\CoPRAM};

\addplot plot coordinates {
(200,     0)
(400,     0)
(600,     0)
(800,     0.06)
(1000,    0.58)
(1200,    0.7)
(1400,    0.9)
(1600,    1)
(1800,    1)
(2000,    1)
};
\addlegendentry{ThWF};

\addplot[color=green,mark=*] plot coordinates {
(200,     0)
(400,     0.08)
(600,     0.66)
(800,     0.88)
(1000,    0.98)
(1200,    1)
(1400,    1)
(1600,    1)
(1800,    1)
(2000,    1)
};
\addlegendentry{SPARTA};

%\legend{CoPRAM\\Block CoPRAM\\ThWF\\SPARTA\\}

\end{axis}
\end{tikzpicture}}
	\subfloat[][$s=30$]{
	\begin{tikzpicture}
\begin{axis}
[width=0.4\textwidth,
xlabel= Number of samples $m$, 
ylabel= Probability of recovery,
grid style = dashed,
grid=both,
legend style={at={(1.2,1.2)}, anchor=north east, legend columns=-1,
} 
]

\addplot[color=blue,mark=square*] plot coordinates {
	(200,     0)
	(400,     0)
	(600,     0.04)
	(800,     0.34)
	(1000,    0.58)
	(1200,    0.82)
	(1400,    1)
	(1600,    1)
	(1800,    1.0)
	(2000,    1.0)
};

\addplot plot coordinates {
(200,    0.02)
(400,    0.38)
(600,    0.78)
(800,    0.96)
(1000,   1)
(1200,   0.98)
(1400,   1)
(1600,   1)
(1800,   1)
(2000,   1)
};

\addplot plot coordinates {
(200,     0)
(400,     0)
(600,     0)
(800,     0)
(1000,    0.18)
(1200,    0.38)
(1400,    0.7667)
(1600,    0.96)
(1800,    0.9333)
(2000,    0.98)
};

\addplot[color=green,mark=*] plot coordinates {
(200,     0)
(400,     0)
(600,     0.02)
(800,     0.18)
(1000,    0.46)
(1200,    0.74)
(1400,    0.8333)
(1600,    1)
(1800,    0.9667)
(2000,    1)
};

%\legend{$CoPRAM$\\$BCoPRAM$\\$ThWF$\\$SPARTA$\\}

\end{axis}
\end{tikzpicture}}
\caption{\sl Phase transition graph for (left) $s=20$ and (right) $s=30$, for a signal of length $n=3,000$ and block length $b=5$.} \label{fig:phase_graph1}
\end{figure}
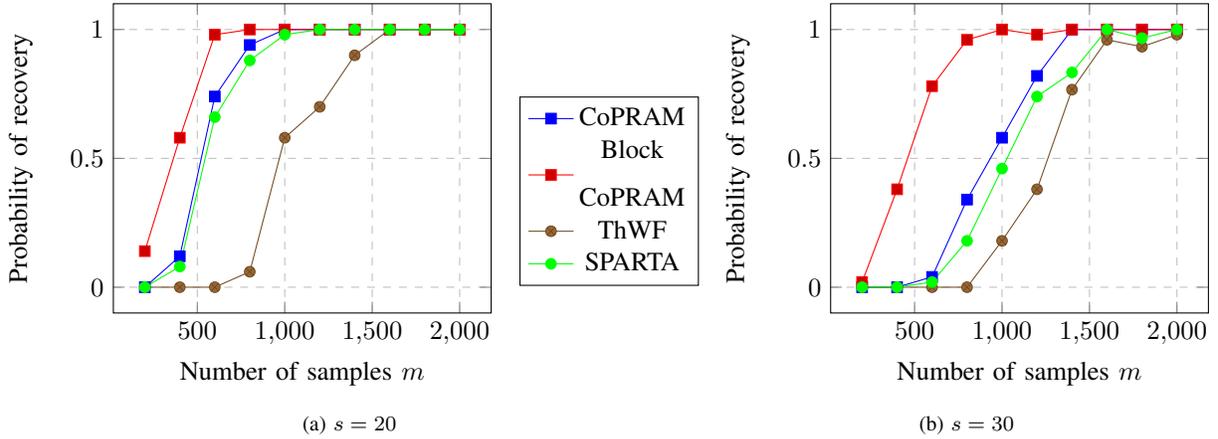
It can be noted that increasing the sparsity of signal shifts the phase transitions to the right (sample complexity of $m$ has a quadratic dependence on $s$ for CoPRAM, SPARTA and Thresholded WF). However the phase transition for Block CoPRAM has a less apparent shift, as compared to other algorithms (sample complexity of $m$ has sub-quadratic dependence on $s$). It can be noted that as sparsity $s$ increases, the \emph{gap} between the phase transition of Block CoPRAM and other algorithms in consideration, increases. As demonstrated in Figure \ref{fig:phase_trans}, the Block CoPRAM approach exhibits lowest sample complexity for the phase transitions in both cases (a) and (b) of Figure \ref{fig:phase_graph1}.

The mean running time of the algorithms for different algorithms is tabulated in Table~\ref{tab:running_time}. It can be noted that the running times of our algorithms CoPRAM and Block CoPRAM are at par with SPARTA and Thresholded WF.
\begin{table}[!t]
	\centering
	\caption{\sl Mean running time of different algorithms at $s=25$.} \label{tab:running_time}
	\begin{tabular}{|c|c|c|c|c|}
		\hline 
		  Algorithm        &  CoPRAM      & Block CoPRAM & SPARTA & ThWF \\
		\hline 
		$m$ at phase trans & 1,600 & 1,400 & 1,800 & 2,000 \\
		\hline 
		mean run time (s) & 0.4000 & 0.3258 & 0.3080 & 0.5808 \\
		\hline
	\end{tabular} 
\end{table}

\textbf{\emph{Leveraging block-sparsity.}} For the second experiment, we study the variation of phase transition with block-length, for Block CoPRAM (refer Figure~\ref{fig:phase_trans_blk}). For this experiment we fixed a signal of length $n=3,000$, sparsities $s = 20, k = 1$ for a block length of $b = 20$. We observed that the phase transitions improve with increase in block length (used to estimate the signal in the algorithm) up to a point. At block sparsities $\frac{s}{b} = \frac{20}{10}=2$ and $\frac{s}{b} = \frac{20 }{20} = 1$, there is little difference in the phase transitions, as the regime of the experiment is very close to the information theoretic bound of $s\log \frac{n}{s}$. 
\begin{figure}[!t]
	\centering
		\begin{tikzpicture}
\begin{axis}
[width=0.4\textwidth,
xlabel= Number of samples $m$, 
ylabel= Probability of recovery,
grid style = dashed,
grid=both,
legend style=
{at={(1.8,0.15)}, 
anchor=south east, 
} ,
]

\addplot plot coordinates {
	(50,     0.025)
	(100,    0.17)
	(150,    0.28)
	(200,    0.435)
	(250,    0.635)
	(300,    0.69)
	(350,    0.745)
	(400,    0.805)
	(500,    0.905)
	(600,    0.96)
	(800,    0.995)
	(1000,   1)
	(1200,   1)
	(1400,   1)
};

\addplot plot coordinates {
(50,     0.005)
(100,    0.11)
(150,    0.22)
(200,    0.45)
(250,    0.655)
(300,    0.695)
(350,    0.81)
(400,    0.925)
(500,    0.945)
(600,    0.98)
(800,    1)
(1000,   1)
(1200,   1)
(1400,   1)
};

\addplot plot coordinates {
	(50,     0)
	(100,    0)
	(150,    0.03)
	(200,    0.135)
	(250,    0.33)
	(300,    0.46)
	(350,    0.675)
	(400,    0.76)
	(500,    0.895)
	(600,    0.95)
	(800,    0.985)
	(1000,   0.995)
	(1200,   1)
	(1400,   1)
};

\addplot[color=orange,mark=square*] plot coordinates {
	(50,     0)
	(100,    0)
	(150,    0)
	(200,    0)
	(250,    0.055)
	(300,    0.12)
	(350,    0.23)
	(400,    0.345)
	(500,    0.645)
	(600,    0.775)
	(800,    0.965)
	(1000,   0.99)
	(1200,   1)
	(1400,   1)
};

\addplot[color=green,mark=*] plot coordinates {
	(50,     0)
	(100,    0)
	(150,    0)
	(200,    0)
	(250,    0)
	(300,    0)
	(350,    0.03)
	(400,    0.035)
	(500,    0.18)
	(600,    0.37)
	(800,    0.815)
	(1000,   0.95)
	(1200,   0.99)
	(1400,   1)
};

\legend{$b=20, k=1$\\$b=10, k=2$\\$b=5, k=4$\\$b=2, k=10$\\ $ b=1, k=20$\\}

\end{axis}
\end{tikzpicture}
	\caption{\sl Variation of phase transition for Block CoPRAM at $s=20$ and $b=20,10,5,2,1$ for a signal of length $n=3,000$.} \label{fig:phase_trans_blk}
\end{figure}
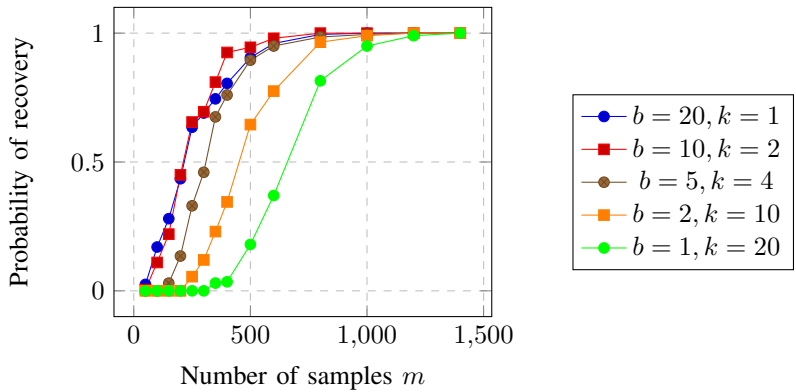

\textbf{\emph{Effect of noise.}}
For our third experiment, we study the effect of noise on the measurements of the form:
\begin{align*}
y_i = \abs{\iprod{\ai}{\xo}} + \epsilon_i,
\end{align*}
for $i\in \{1,2\dots m\}$, to verify the claims in Theorem \ref{thm:noise}. The noise vector $\epsilon \in \reals^m$ is sampled from a zero-mean Gaussian distribution $\gauss(0,\sigma^2)$, where $\sigma^2$ is determined using the input noise-signal-ratio (NSR). We compared CoPRAM, Block CoPRAM and SPARTA to analyze robustness to noisy measurements for amplitude only measurements (ThWF is excluded because they use quadratic measurements). We vary the input NSR $= \sigma^2/\twonorm{\xo}^2$, from $0.1$ to $1$ in steps of $0.1$. We fix signal parameters $n=3,000, s=20, b=5, k=4$ and number of measurements to $m=1,600$. This experiment was run for $50$ independent Monte Carlo trials. The variation of mean relative error $\twonorm{\x^{t_0} - \xo}/\twonorm{\xo}$ can be seen in Figure \ref{fig:noise_err}. We observe that Block CoPRAM exhibits greater robustness to noise compared to CoPRAM and SPARTA in all cases considered. 
\begin{figure}[!t]
	\centering
	\begin{tikzpicture}
\begin{axis}
[width=0.4\textwidth,
xlabel= Noise-to-signal ratio NSR, 
ylabel= Relative error in recovery,
grid style = dashed,
grid=both,
legend style={at={(1.8,0.3)}, anchor=south east, %legend columns=-1,
} 
]

\addplot[color=blue,mark=square*] plot coordinates {
(0.1, 0.1133)
(0.2, 0.1619)
(0.3, 0.2014)
(0.4, 0.2330)
(0.5, 0.2592)
(0.6, 0.2924)
(0.7, 0.3148)
(0.8, 0.3400)
(0.9, 0.4064)
(1,   0.4068)
};
\addlegendentry{CoPRAM};

\addplot plot coordinates {
	(0.1,     0.0391)
	(0.2,     0.0541)
	(0.3,     0.0633)
	(0.4,     0.0788)
	(0.5,     0.0897)
	(0.6,     0.1080)
	(0.7,     0.1822)
	(0.8,     0.1655)
	(0.9,     0.1904)
	(1,       0.2597)
};
\addlegendentry{Block CoPRAM};

\addplot[color=green,mark=*] plot coordinates {
(0.1, 0.1096)
(0.2, 0.1625)
(0.3, 0.2122)
(0.4, 0.2641)
(0.5, 0.3246)
(0.6, 0.6233)
(0.7, 0.9654)
(0.8, 1.0102)
(0.9, 1.0141)
(1,   1.0214)
};
\addlegendentry{SPARTA};

\end{axis}
\end{tikzpicture}
	\caption{\sl Variation of mean relative error in signal recovered v/s input NSR at $s=20$ and $b=5,k=4$ for a signal of length $n=3,000$, and number of measurements $m=1,600$.} \label{fig:noise_err}
\end{figure}
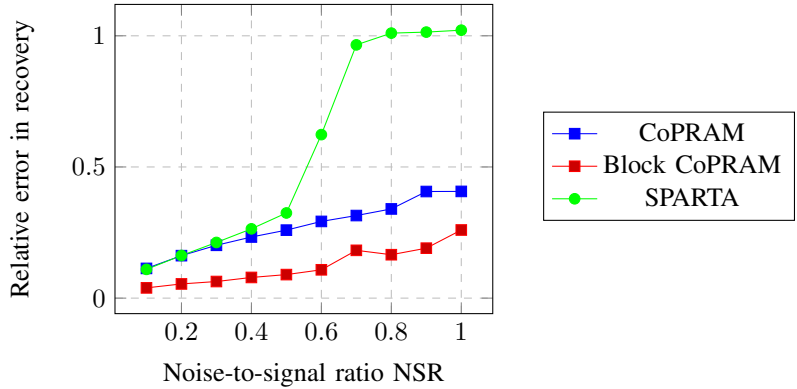

\begin{figure}[!h]
	\centering
\subfloat[][$s=20$]{
	\begin{tikzpicture}
\begin{axis}
[width=0.39\textwidth,
xlabel= Number of samples $m$, 
ylabel= Probability of recovery,
grid style = dashed,
grid=both,
legend style=
{at={(1.7,0.2)}, 
anchor=south east, 
} ,
]

\addplot plot coordinates {
	(200,     0)
	(400,    0.0200)
	(600,    0.6800)
	(800,    0.9000)
	(1000,   0.9600)
	(1200,    1)
	(1400,    1)
	(1600,    1)
	(1800,    1)
	(2000,    1)
	(2200,    1)
};

\addplot plot coordinates {
	(200,     0)
	(400,    0.3400)
	(600,    0.9000)
	(800,    0.9800)
	(1000,    1)
	(1200,    1)
	(1400,    1)
	(1600,    1)
	(1800,    1)
	(2000,    1)
	(2200,    1)
};

\addplot[mark=*, color=green] plot coordinates {
	(200,     0)
	(400,    0.8200)
	(600,     1)
	(800,     1)
	(1000,    1)
	(1200,    1)
	(1400,    1)
	(1600,    1)
	(1800,    1)
	(2000,    1)
	(2200,    1)
};

\addplot plot coordinates {
	(200,     0)
	(400,     0.96)
	(600,     1)
	(800,     1)
	(1000,    1)
	(1200,    1)
	(1400,    1)
	(1600,    1)
	(1800,    1)
	(2000,    1)
	(2200,    1)
};

\legend{Random normal\\$\alpha = 2$\\$\alpha = 4$\\$\alpha = 8$\\}

\end{axis}
\end{tikzpicture}}
\subfloat[][$s=30$]{
	\begin{tikzpicture}
\begin{axis}
[width=0.39\textwidth,
xlabel= Number of samples $m$, 
ylabel= Probability of recovery,
grid style = dashed,
grid=both,
legend style=
{at={(1.5,0)}, 
anchor=south east, 
} ,
]

\addplot plot coordinates {
	(200,    0)
	(400,    0)
	(600,    0)
	(800,    0.2200)
	(1000,   0.5200)
	(1200,   0.84)
	(1400,   0.92)
	(1600,   0.98)
	(1800,    1)
	(2000,    1)
	(2200,    1)
};

\addplot plot coordinates {
	(200,     0)
	(400,     0)
	(600,    0.7200)
	(800,    0.9600)
	(1000,    1)
	(1200,    1)
	(1400,    1)
	(1600,    1)
	(1800,    1)
	(2000,    1)
	(2200,    1)
};

\addplot[mark=*, color=green] plot coordinates {
	(200,     0)
	(400,    0.1200)
	(600,    0.88)
	(800,     1)
	(1000,    1)
	(1200,    1)
	(1400,    1)
	(1600,    1)
	(1800,    1)
	(2000,    1)
	(2200,    1)
};

\addplot plot coordinates {
	(200,     0)
	(400,     0.08)
	(600,     0.98)
	(800,     1)
	(1000,    1)
	(1200,    1)
	(1400,    1)
	(1600,    1)
	(1800,    1)
	(2000,    1)
	(2200,    1)
};

%\legend{No decay\\$\alpha = 2$\\$\alpha = 4$\\$\alpha = 8$\\}

\end{axis}
\end{tikzpicture}}
	\caption{\sl Variation of phase transition for CoPRAM at (a) $s=20$ and (b) $s=30$ at decay rates $\alpha=2,4,8$, for a signal of length $n=3,000$, compared to standard $s$-sparse signal with coefficients picked random normally.} \label{fig:phase_trans_pl}
\end{figure}
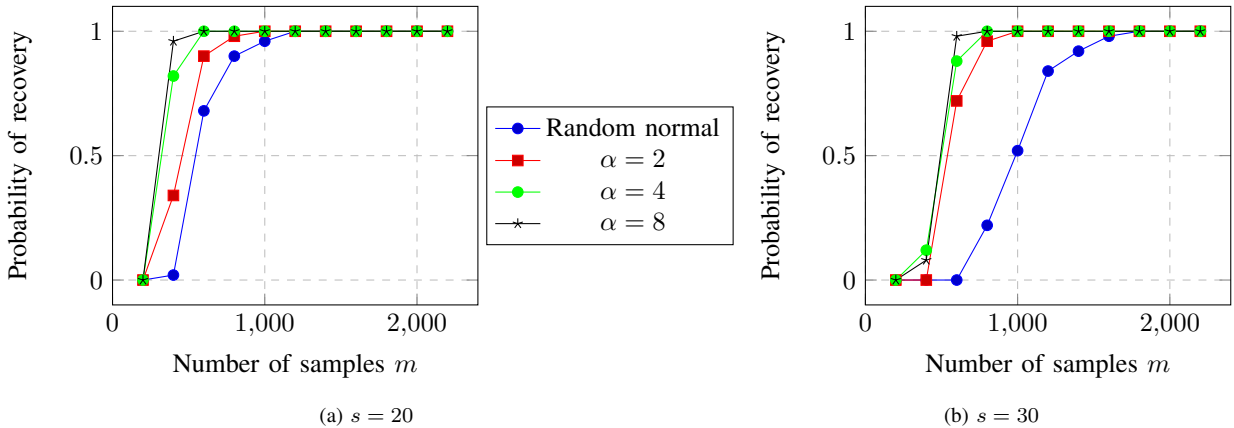

\textbf{\emph{Power-law decay.}} For our fourth experiment, we verify the claims in Theorem \ref{thm:PLdecay}. We set the signal length to $n=3,000$. We analyze the effect of power law decay on signals with sparsities $s=20$ and $s=30$ for different rates of decay $\alpha = 2,4,8$ and compare this to the case with no power law decay (coefficients picked random normally). We observe an improvement in sample complexity with respect to the "no powerlaw decay" case, as seen in Figure \ref{fig:phase_trans_pl}. The improvement is more prominent as we increase the sparsity from $s=20$ to $s=30$.

\textbf{\emph{Experiments on real images.}} For our final experiment, we evaluated the performance of our algorithm on a real image, with induced sparsity in the wavelet basis (db1). We used a $128 \times 128$ image of Lovett Hall, and used the thresholded wavelet transform (using Haar wavelet) of this image as the sparse signal with $s = 0.09n$. This image was reconstructed using $m=16,384$ samples, using CoPRAM and the standard AltMinPhase algorithm described in \cite{netrapalli}. We used signal-to-noise ratio (SNR) to quantify quality of reconstruction. In Figure \ref{fig:lovett}, we demonstrate how enforcing a sparsity constraint enables recovery of the same image using fewer samples and better reconstruction quality.

\begin{figure}[!h]
	\centering
		\subfloat[][Original]{
		\includegraphics[width = 0.25\textwidth]{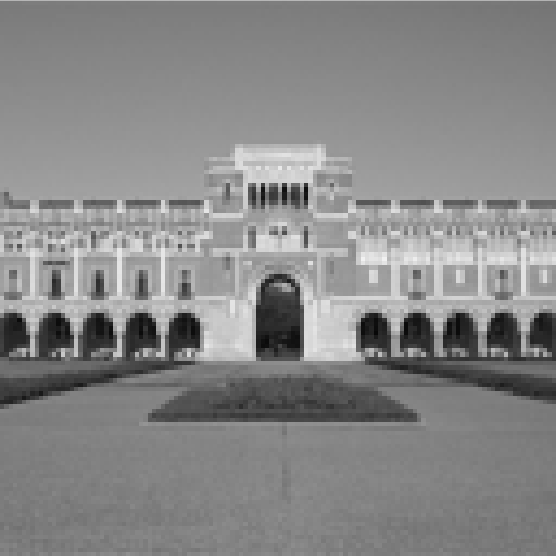}} \quad\quad \quad
	\subfloat[][AltMinPhase, SNR=-0.71dB]{
		\includegraphics[width = 0.25\textwidth]{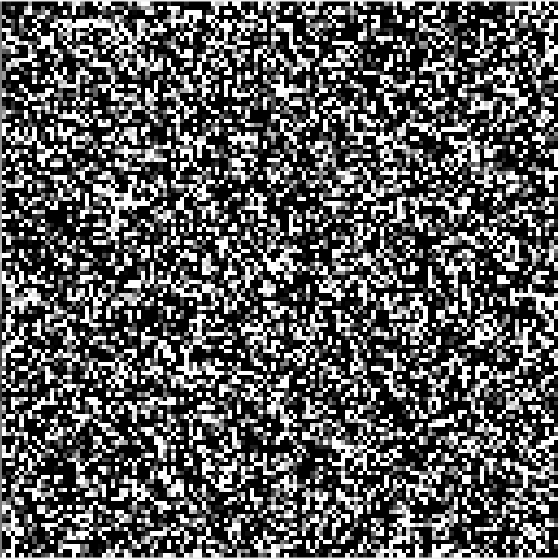}} \quad\quad\quad
		\subfloat[][CoPRAM, SNR=82.86dB]{
		\includegraphics[width = 0.25\textwidth]{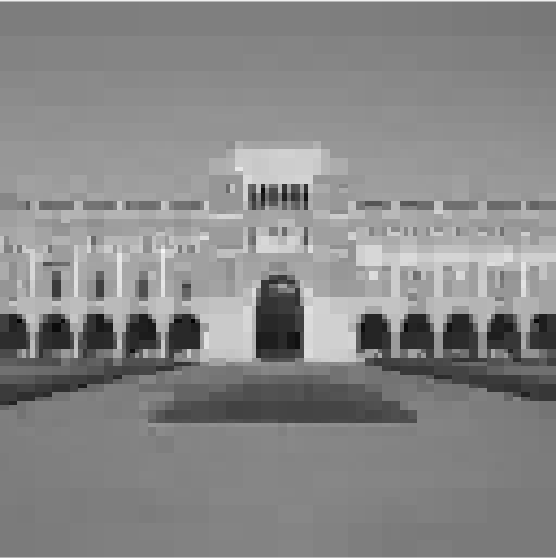}} 
	\caption{\sl Reconstruction of the original Lovett (a) image using (b) AltMinPhase for $m=16,384$, (c) CoPRAM for $m=16,384$ measurements, where $f=m/n$.} \label{fig:lovett}
\end{figure}

%\clearpage

\section{Discussion} \label{sec:remarks}

In this paper, we have introduced a set of new algorithmic approaches for sparse as well as structured sparse phase retrieval. Our algorithms are conceptually simple and indeed are reminiscent of classical heuristics for phase retrieval; however, our analysis also shows an asymptotic reduction in sample complexity when additional structures on top of standard sparsity are leveraged within the reconstruction process. Several open questions remain, including expanding our analysis to more sophisticated sparsity models, such as clusters, trees, groups, and graphical models~\cite{surveyEATCS}. Moreover, our analysis only applies to the case of Gaussian samples, and extending our results to more realistic  measurement schemes such as Fourier samples and coded diffraction patterns~\cite{candes2015phase} will be an interesting direction of future study; indeed, our preliminary experimental results in~\cite{sfp,lrfp} have shown the empirical benefits of our algorithms for such challenging measurement setups.

%Acknowledgement
%\section*{Acknowledgment}

\appendices

%Appendix A - proofs of initialization
%Appendix A

\section{CoPRAM Initialization} \label{sec:appendixA}
In this section we state the proofs related to the \textit{initialization} in Algorithm~\ref{algo:initial_copram}, for compressive phase retrieval. This includes the proofs of Lemmas~\ref{lem:smallmarginals}~-~\ref{lem:xin} which complete the proof of Theorem~\ref{thm:main_initial}.\\

The outline of the proof is sketched out as follows. Using Lemma \ref{lem:smallmarginals}, we can find an upper bound on marginals $M_{jj}$ for $j\in S$. Consequently,
\begin{align} \label{eq:margin_ub}
\max_{j\in S^c} M_{jj} \leq \rbrak{1+11 \sqfr{\log mn}{m}} \twonorm{\xo}^2 = \Theta_1
\end{align}
with probability greater than $1-\frac{5}{m}$.
Marginals $M_{jj}$ for $j\in S$ can be evaluated in two ways:
\begin{enumerate}
	\item Assuming a bound on the minimum element of $\xo$: $x_{min}^{*2} \equiv \min_{j\in S} x_j^{*2} = \frac{C}{s}\twonorm{\xo}^2$. The proof then carries forward from the work in \cite{sparta}, where they arrive at the lower bound on the minimum marginal for $j\in S$, with probability greater than $1-\frac{1}{m}$,
	\begin{align*}
	\min_{j\in S} M_{jj} &\geq \twonorm{\xo}^2 + x_{min}^{*2} \\
	&= \left(1+\frac{C}{s}\right)\twonorm{\xo}^2 = \Theta_2,
	\end{align*}
	given that $m \geq C_0 s^2 \log(mn)$. This proof is similar to that mentioned in Lemma~\ref{lem:largemarginals}.  Piecing these two together,
	\begin{align*}
	\min_{j\in S} M_{jj} &\geq \left(1+\frac{C}{s}\right)\twonorm{\xo}\\
	 &> \left(1+11 \sqfr{\log mn}{m} \right)\twonorm{\xo}^2\\
	 &\geq \max_{j\in S^c} M_{jj} .
	\end{align*}
	which implies that the support picked up using the top $s$-marginals $M_{jj}$ is the true support with probability greater than $1-\frac{6}{m}$, given $m\geq C_0 s^2 \log(mn)$, as long as there is a clear separation between $\Theta_1$ and $\Theta_2$. They proceed to show that with a high probability, $\dist{\xin}{\xo} \leq \delta_0\twonorm{\xo}$, using Proposition 1 of \cite{TAF}, completing the proof of Theorem~\ref{thm:main_initial}.
	\item If there is no such assumption on the minimum entry $x_{min}^{*2}$, we proceed with a longer proof, as stated below using Lemmas \ref{lem:largemarginals}-\ref{lem:xin}. The idea is to show that $\xo \approx \xShat$ and subsequently $\xShat \approx \xin $, effectively implying that $\xin \approx \xo$.
\end{enumerate}

This idea and the partition of support sets used in the proof have been illustrated in Figures \ref{fig:illus1} and \ref{fig:illus2}.

\begin{figure}[h]
	\centering
	%%illustration - sets

%\begin{tikzpicture}
%\draw [thin, fill = orange, fill opacity=0.1]
%%rectangle (1,1);
%(0,0) -- (0,1) -- (1,1) -- (1,0) -- cycle;
%\end{tikzpicture}

\begin{tikzpicture}[scale=1]
\centering
\draw [thin, fill = orange, fill opacity=0.2]
(0,0) -- (0,2) -- (6,2) -- (6,0) -- (0,0) ;
\node at (3,1) {$|S^c| = n-s$};
\draw [green, thick, dashed]
(6,2) -- (7,2);
\draw [green, thick, dashed]
(6,0) -- (7,0);
\draw [thin, fill = blue, fill opacity=0.2]
(7,0) -- (9,0) -- (9,2) -- (7,2) -- (7,0) ; 
\node at (8,1) {$|S| = s$};
\draw [red, thick, dashed]
(6,0) -- (6,2); 
\draw [yellow, thick, dashed]
(7,0) -- (7,2); 
\node at (0,3) {low $M_{jj}$};
\draw [->] (0,2.8) -- (0,2.2);
\node at (9,3) {high $M_{jj}$};
\draw [->] (9,2.8) -- (9,2.2);
\node at (7,3) {$\Theta_2$};
\draw [->] (7,2.8) -- (7,2.2);
\node at (6,3) {$\Theta_1$};
\draw [->] (6,2.8) -- (6,2.2);
\draw[decoration={brace,mirror,raise=5pt},decorate]
(0,0) -- node[below=6pt] {bottom $(n-s)$ marginals} (5.9,0);
\draw[decoration={brace,mirror,raise=5pt},decorate]
(6.1,0) -- node[below=6pt] {top $s$ marginals} (9,0);
\end{tikzpicture}
	\caption{Partition of supports considered for analysis of proof approach 1: assumption on $x_{min}^*$.} \label{fig:illus1}
\end{figure}
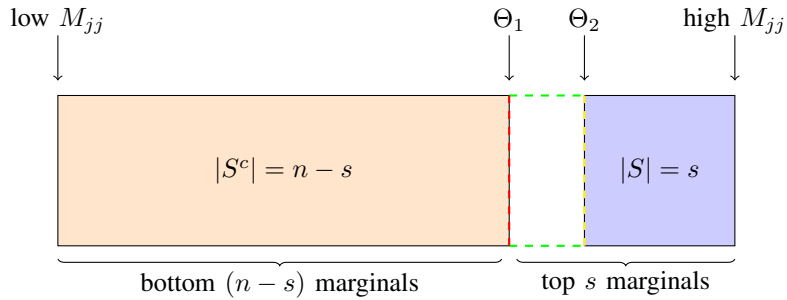

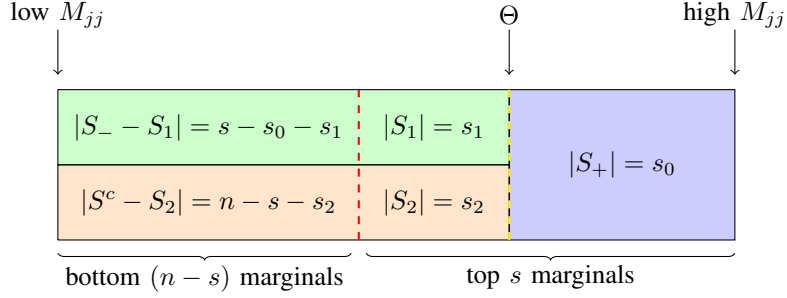
\begin{figure}[h]
	\centering
	%%illustration - sets

%\begin{tikzpicture}
%\draw [thin, fill = orange, fill opacity=0.1]
%%rectangle (1,1);
%(0,0) -- (0,1) -- (1,1) -- (1,0) -- cycle;
%\end{tikzpicture}

\begin{tikzpicture}[scale=1]
\centering
\draw [thin, fill = orange, fill opacity=0.2]
(0,0) -- (0,1) -- (6,1) -- (6,0) -- (0,0) ;
\node at (2,0.5) {$|S^c-S_2| = n-s-s_2$};
\node at (5,0.5) {$|S_2| = s_2$};
\draw [thin, fill = green, fill opacity=0.2]
(0,1) -- (6,1) -- (6,2) -- (0,2) -- (0,1) ; 
\node at (2,1.5) {$|S_{-}-S_1| = s-s_0 - s_1$};
\node at (5,1.5) {$|S_1| = s_1$};
\draw [thin, fill = blue, fill opacity=0.2]
(6,0) -- (9,0) -- (9,2) -- (6,2) -- (6,0) ; 
\node at (7.5,1) {$|S_+| = s_0$};
\draw [red, thick, dashed]
(4,0) -- (4,2); 
\draw [yellow, thick, dashed]
(6,0) -- (6,2); 
\node at (0,3) {low $M_{jj}$};
\draw [->] (0,2.8) -- (0,2.2);
\node at (9,3) {high $M_{jj}$};
\draw [->] (9,2.8) -- (9,2.2);
\node at (6,3) {$\Theta$};
\draw [->] (6,2.8) -- (6,2.2);
\draw[decoration={brace,mirror,raise=5pt},decorate]
(0,0) -- node[below=6pt] {bottom $(n-s)$ marginals} (3.9,0);
\draw[decoration={brace,mirror,raise=5pt},decorate]
(4.1,0) -- node[below=6pt] {top $s$ marginals} (9,0);
\end{tikzpicture}
	\caption{Partition of supports considered for analysis of proof approach 2.} \label{fig:illus2}
\end{figure}

\begin{lemma} \label{lem:smallmarginals}
	For all $j\in S^c$, with probability greater than $1-\frac{5}{m}$, the corresponding marginals are upper-bounded as
	\begin{align}
	M_{jj} \leq \rbrak{1 + 11\sqrt{\frac{\log{mn}}{m}}}\twonorm{\xo}^2 = \Theta.
	\end{align}
\end{lemma}

\begin{proof}
	Evaluating the marginals:
	\begin{align*}
	M_{jj} -\phi^2= \frac{1}{m}\sum_{i=1}^{m} y_i^2\rbrak{a_{ij}^2 - 1},
	\end{align*}
	where $y_i$ is independent of $a_{ij}$ for all $j \in S^c$. Evaluating the tail bound in terms of a series of tail bounds for independent random variables $y_i$ and $a_{ij}$, one can use Lemma 4.1 of \cite{chisquare} for the $\chi_1^2$ variables $a_{ij}^2$ with weights $y_i^2$ (here $p \equiv n-s$):
%	\begin{gather*}
%	\prob{\sum_{i=1}^{m} y_i^2\rbrak{a_{ij}^2 - 1} > 2\sqrt{t} \rbrak{\sum_{i=1}^{m} y_i^4}^{\frac{1}{2}} + 2\rbrak{\max_i y_i^2}t}\\
%	\leq \exp(-t) = \frac{1}{mp}.
%	\end{gather*}
	\begin{gather*}
\prob{\sum_{i=1}^{m} y_i^2\rbrak{a_{ij}^2 - 1} > 2\sqrt{t} \rbrak{\sum_{i=1}^{m} y_i^4}^{\frac{1}{2}} + 2\rbrak{\max_i y_i^2}t}
\leq \exp(-t) = \frac{1}{mp}.
\end{gather*}
	Further, using the Chebyshev's inequality for $y_i^2$:
	\begin{align*}
	\prob{\sum_{i=1}^{m} \frac{y_i^4}{\twonorm{\xo}^4} > 3m + \sqrt{96m} t} \leq \frac{1}{t^2}  = \frac{1}{mp}. 
	\end{align*} 
	Using the Gaussian tail bound for $y_i^2$ followed by union bound:
	\begin{align*}
	\prob{\max_i \frac{y_i^2}{\twonorm{\xo}^2} > t} \leq 2m\exp\rbrak{\frac{-t}{2}} = \frac{2}{mp^2} \leq \frac{2}{mp}.
	\end{align*}
	With probability at most $\frac{4}{mp}$, for each $j\in S^c$, using a union bound on these three tail bounds,
%	\begin{align*}
%	\frac{1}{m} \sum_{i=1}^{m} y_i^2 (a_{ij}^2 - 1) &> 2\sqrt{3+\sqrt{96p}} \twonorm{\xo}^2 \sqfr{\log mp}{m} \\
%	&+ 8 \twonorm{\xo}^2 \frac{\rbrak{\log mp}^2}{m},\\
%	&> 2\sqrt{3+\sqrt{96}} \twonorm{\xo}^2 \sqfr{\log mp}{m} \\
%	&+ 8 \twonorm{\xo}^2 \frac{\rbrak{\log mp}^2}{m}.
%	\end{align*}
	\begin{align*}
	\frac{1}{m} \sum_{i=1}^{m} y_i^2 (a_{ij}^2 - 1) &> 2\sqrt{3+\sqrt{96p}} \twonorm{\xo}^2 \sqfr{\log mp}{m} + 8 \twonorm{\xo}^2 \frac{\rbrak{\log mp}^2}{m},\\
	&> 2\sqrt{3+\sqrt{96}} \twonorm{\xo}^2 \sqfr{\log mp}{m} + 8 \twonorm{\xo}^2 \frac{\rbrak{\log mp}^2}{m}.
	\end{align*}
	Using a union bound for all $j \in S^c$ ($p$ such), with probability at least $1-\frac{4}{m}$, 
%	\begin{align*} 
%	\frac{1}{m} \sum_{i=1}^{m} y_i^2 (a_{ij}^2 - 1) &\leq 2\sqrt{3+\sqrt{96}} \twonorm{\xo}^2 \sqfr{\log mp}{m}\\ 
%	 &+ 8 \twonorm{\xo}^2 \frac{\rbrak{\log mp}^2}{m},\\ 
%	 &\leq 8\sqfr{\log mp}{m}\twonorm{\xo}^2.
%	\end{align*}
	\begin{align*} 
\frac{1}{m} \sum_{i=1}^{m} y_i^2 (a_{ij}^2 - 1) &\leq 2\sqrt{3+\sqrt{96}} \twonorm{\xo}^2 \sqfr{\log mp}{m} + 8 \twonorm{\xo}^2 \frac{\rbrak{\log mp}^2}{m},\\ 
&\leq 8\sqfr{\log mp}{m}\twonorm{\xo}^2.
\end{align*}
	Using Lemma \ref{lem:signalpower}, for $m > C$, and using the fact that $p \leq n$: 
	\begin{align} \nonumber
	M_{jj} &= \frac{1}{m} \sum_{i=1}^{m} y_i^2 a_{ij}^2 \leq 8 \sqfr{\log mn}{m} \twonorm{\xo}^2 + \phi^2,\nonumber\\ \label{eq:marginal_ub}
	M_{jj} &\leq \rbrak{1+11\sqfr{\log mn}{m}}\twonorm{\xo}^2 = \Theta,
	\end{align}
	which establishes the upper bound on marginals associated with the zero-locations $j \in S^c$,  with probability greater than $1-\frac{5}{m}$.
\end{proof}

\begin{lemma}
	\label{lem:largemarginals}
	For $j\in S_+ \subset S$, with probability greater than $1-\frac{2}{m}$, the corresponding marginals are lower-bounded as
	\begin{align} \label{eq:marg_lb}
	M_{jj} \geq \rbrak{1 + 11\sqrt{\frac{\log{mn}}{m}}}\twonorm{\xo}^2 = \Theta,
	\end{align}
	where $S_{+}$ is defined as
	\begin{align} \label{eq:largeelements}
	S_{+} = \cbrak{j \in S \mid  x_j^{*2} > {15}\sqfr{\log{mn}}{m}\twonorm{\xo}^2} .
	\end{align}
	Subsequently, we can define $S_{-}$ as
	\begin{align} \label{eq:smallelements}
	S_{-} = \cbrak{j \in S \mid  x_j^{*2} \leq 15\sqfr{\log{mn}}{m}\twonorm{\xo}^2} ,
	\end{align}
	with $S_{+}$ and $S_{-}$ forming a partition of $S$ and the corresponding energy in the elements $x_j, j\in S_{-}$ is lower-bounded as	
	\begin{align} \label{eq:ubsmall}
	\twonorm{\xSmin}^2 \leq 15 \sqfr{s^2 \log{mn}}{m}\twonorm{\xo}^2.
	\end{align}
\end{lemma}

\begin{proof}
	Evaluating the marginals:
	\begin{align} \label{eq:marg}
	M_{jj} -\phi^2= \frac{1}{m}\sum_{i=1}^{m} y_i^2\rbrak{a_{ij}^2 - 1}.
	\end{align}
	For $j\in S$, $y_i$ and $a_{ij}$ are dependent random variables. The marginal $M_{jj}$ can be evaluated through a concentration bounds on the two terms that compose the RHS of \eqref{eq:marg}: $\frac{1}{m} \sum_{i=1}^m y_i^2 a_{ij}^2$ and $\frac{1}{m} \sum_{i=1}^m y_i^2$. This can be done by evaluating the expectation values:
%	\begin{align*}
%	\expec{y_i^2} &= \twonorm{\xo}^2,\\
%	\expec{y_i^2 a_{ij}^2} &= \twonorm{\xo}^2 + 2 x_j^2, \\
%	\expec{y_i^4a_{ij}^4} &= 105x_j^4 + 90x_j^2\rbrak{\twonorm{\xo}^2 - x_j^2}\\ &+ 9\rbrak{\twonorm{\xo}^2 - x_j^2}^2.
%	\end{align*}
	\begin{align*}
\expec{y_i^2} &= \twonorm{\xo}^2,\\
\expec{y_i^2 a_{ij}^2} &= \twonorm{\xo}^2 + 2 {x_j^*}^2, \\
\expec{y_i^4a_{ij}^4} &= 105{x_j^*}^4 + 90{x_j^*}^2\rbrak{\twonorm{\xo}^2 - {x_j^*}^2} + 9\rbrak{\twonorm{\xo}^2 - {x_j^*}^2}^2.
\end{align*}
	Constructing variable $X_i = \twonorm{\xo}^2 + 2{x_j^*}^2 - y_i^2 a_{ij}^2$ which is upper bounded, with zero mean and bounded variance, we can use Lemma~\ref{lem:martingale} to establish a concentration bound with parameters:
	\begin{align*}
	X_i &\leq \twonorm{\xo}^2 + 2{x_j^*}^2 \leq 3\twonorm{\xo}^2,\\
	\expec{X_i} &= 0,\\
	\expec{X_i^2} &= 20{x_j^*}^4 + 68\twonorm{\xo}^2 {x_j^*}^2 + 8\twonorm{\xo}^4 \leq 96\twonorm{\xo}^4.
	\end{align*} 
	Using Lemma~\ref{lem:martingale}, for each $j\in S$,
%	\begin{align} \nonumber
%	&\prob{\sum_{i=1}^m -X_i \leq -t}\\ \nonumber 
%	= &\prob{\sum_{i=1}^m y_i^2 a_{ij}^2 - m\rbrak{\twonorm{\xo}^2 + 2x_j^2}   \leq -t},\\  \label{eq:conc1}
%	&\leq \exp\rbrak{-\frac{t^2}{192\twonorm{\xo}^4m}} \leq \frac{1}{mk}.
%	\end{align}
	\begin{align} \nonumber
\prob{\sum_{i=1}^m -X_i \leq -t} &= \prob{\sum_{i=1}^m y_i^2 a_{ij}^2 - m\rbrak{\twonorm{\xo}^2 + 2{x_j^*}^2}   \leq -t},\\  \label{eq:conc1}
&\leq \exp\rbrak{-\frac{t^2}{192\twonorm{\xo}^4m}} \leq \frac{1}{mk}.
\end{align}
	This requires $t = \sqrt{192} \twonorm{\xo}^2 \sqrt{m \log mk} \approx 13.86 \twonorm{\xo}^2 \sqrt{m \log mk} \leq 13.86 \twonorm{\xo}^2 \sqrt{m \log mn}$. This establishes the bound on the first term  $\frac{1}{m} \sum_{i=1}^m y_i^2 a_{ij}^2$. Similarly, we can establish a bound on the second term $\frac{1}{m} \sum_{i=1}^m y_i^2$, which requires Lemma 4.1 of \cite{chisquare}, with probability greater than $1-\frac{1}{mk}$, for each $j \in S$: 
	\begin{align} \nonumber
	\frac{1}{m} \sum_{i=1}^m y_i^2 - \twonorm{\xo}^2 &\leq \rbrak{2\sqfr{\log mk}{m} + \frac{2\log mk}{m}}\twonorm{\xo}^2,\\ \nonumber &\leq 3 \twonorm{\xo}^2 \sqfr{\log mk}{m},\\ &\leq 3 \twonorm{\xo}^2 \sqfr{\log mn}{m}. \label{eq:conc2}
	\end{align}
	for $m > C$. Combining these two concentration bounds \eqref{eq:conc1}, \eqref{eq:conc2}, taking a union bound for all $j \in S_+$ and substituting in \eqref{eq:marg}:
	\begin{align} \label{eq:cconc}
	M_{jj} - \phi^2 \geq 2{x_j^*}^2 - 17 \sqfr{\log mn}{m} \twonorm{\xo}^2,
	\end{align}
	which holds with probability at least $1 - \frac{2}{m}$.\\
	
	If the set $S_{+}$, is constructed as in \eqref{eq:largeelements}, then evaluating the bound in \eqref{eq:cconc}, we get:
	\begin{align*}
	M_{jj} - \phi^2 &\geq 2{x_j^*}^2 - 17 \sqfr{\log mn}{m} \twonorm{\xo}^2,\\
	M_{jj} &\geq \rbrak{1+ 2{x_j^*}^2 - 19 \sqfr{\log mn}{m}} \twonorm{\xo}^2,\\ &\geq \rbrak{1+11  \sqfr{\log mn}{m}}\twonorm{\xo}^2,
	\end{align*}
	holds for all elements $j\in S_+$, with probability greater than $1 - \frac{2}{m}$, yielding the bound in \eqref{eq:marg_lb}.
\end{proof}

\begin{lemma}
	\label{lem:xShat}
	If $\hat{S}$ is chosen as in Algorithm~\ref{algo:initial_copram}, with probability greater than $1-\frac{2}{m}$,
	\begin{align} \label{eq:xoxShat}
	\twonorm{\xo - \xShat} \leq \delta_1 \twonorm{\xo},
	\end{align}
	as long as the number of measurements $m$ follow the following bound 
	\begin{align} \label{eq:mbound}
	m \geq C s^2 \log mn.
	\end{align}
	
\end{lemma}

\begin{proof}
	If $\hat{S}$ is chosen such that it corresponds to the top-$s$ marginals $M_{jj}$, then it will pick up $S_{+}$ corresponding to large marginals $M_{jj} > \Theta $, $S_1 = S_{-} \cap \hat{S}$ and $S_2 = S^c \cap \hat{S}$ corresponding to small marginals $M_{jj} < \Theta$ ($S_{+},S_1,S_2$ form a partition of $\widehat{S}$ and $\card{\hat{S}} = s$, refer Figure~\ref{fig:illus2} for illustration of the sets):
	\begin{align} \label{eq:decompose}
	\xShat = \xSpls + \x_{S_1}^* + \x_{S_2}^*.
	\end{align}	
	By definition $\x_{S^c} = \vect{0}$ and therefore $\x_{S_2} = \vect{0}$.
	If we can prove that $\xo \approx \xShat$ and $\xShat \approx \xin$, then we can claim that $\xin \approx \xo$. First, we prove that $\twonorm{\xo - \xShat} \leq \delta_1 \twonorm{\xo}$:
	\begin{align*}
	\twonorm{\xo - \xShat}^2 &= \twonorm{\xo - \xSpls - \x_{S_1}^* }^2, \\
	&\leq \twonorm{\xo - \xSpls }^2 + \twonorm{\x_{S_1}^* }^2, \\
	&\leq \twonorm{\xo - \xSpls }^2 + \twonorm{\xSmin }^2.
	\end{align*}
	By construction, $S_{-}$ and $S_{+}$ form a partion of $S$:
	\begin{align*}
	\xo &= \xSmin + \xSpls, \\
	\implies  \twonorm{\xo - \xShat}^2 &\leq 2\twonorm{\xSmin }^2.
	\end{align*} 
	Using \eqref{eq:ubsmall}, we compute the bound,
	\begin{align} \label{eq:loosebound}
	\twonorm{\xo - \xShat}^2 &\leq 30 \sqfr{s^2 \log{mn}}{m} \twonorm{\xo}^2, \\\nonumber&\leq \delta_1^2 \twonorm{\xo}^2.
	\end{align}
	since there are \textit{at most} $s$ elements in $S_-$, which is the required condition \eqref{eq:xoxShat}. This requires sample complexity $m$ to satisfy:
	\begin{gather}
	30 \sqfr{s^2 \log{mn}}{m}  \leq \delta_1^2,\nonumber \\
	\implies 
	m \geq \frac{900}{\delta_1^2} s^2 \log mn = C(\delta_1)  s^2 \log mn .\label{eq:constm} 
	\end{gather}
\end{proof}

We have proved that $\xo \approx \xShat$. Now we need to prove that $\xShat \approx \xin $, which we do using Lemma~\ref{lem:xin}.

\begin{lemma} \label{lem:xin}
	With probability greater than $1-\frac{8}{m}$
	\begin{align*}
	\distop{\xin}{\xShat} &\equiv \min\rbrak{{\twonorm{\xin - \xShat},\twonorm{\xin + \xShat}}},\\
	& \leq \delta_2 \twonorm{\xo},
	\end{align*}
	as long as the number of measurements $m$ follow the following bound 
	\begin{align*}
	m \geq C s \log n.
	\end{align*}
\end{lemma}

\begin{proof}
	The top singular vector of $\expec{\M} $ is equal to true $\xo$, from \eqref{eq:spectral}:
	\begin{align*}
	\expec{\M} &= \expec{\frac{1}{m}\sum_{j=1}^{m} y_j^2 \ai\ai^\top},\\
	 &= \left(\I_{n\times n} + 2\frac{\xo}{\|\xo\|_2}\frac{\xo^\top}{\|\xo\|_2}\right)\|\xo\|_2^2 ,\\
	\text{similarly,}\:\: \expec{\M_{S}} &= \expec{\frac{1}{m}\sum_{i=1}^{m} y_i^2 \aiS \aiS^\top},\\ &= \left((\I_{n\times n})_S + 2\frac{\xo}{\|\xo\|_2}\frac{\xo^\top}{\|\xo\|_2}\right)\|\xo\|_2^2, \\&= \expec{\M}.
	\end{align*}
	We then define $\M_{\hat{S}} = \frac{1}{m}\sum_{i=1}^{m} y_i^2 \aiShat \aiShat^\top$ and $\xin$ is the top singular vector of $\M_{\hat{S}}$.\\
	
	Defining $S_3 \equiv (S \cup S_2) \subset (S\cup\hat{S})$, where $S_2 = \hat{S}\cap S^c$, then, $\card{S_3} \leq 2s$, and,
	\begin{align*}
	\expec{\M_{S_3}} &= \expec{\frac{1}{m}\sum_{i=1}^{m} y_i^2 \ai_{S_3} \ai_{S_3}^\top},\\ &= \left((\I_{n\times n})_{S_3} + 2\frac{\xo}{\|\xo\|_2}\frac{\xo^{\top}}{\|\xo\|_2}\right)\|\xo\|_2^2.
	\end{align*}
	At this stage, we can invoke the proof idea from \cite{cai}, as stated in Lemma \ref{lem:svd_mat} from Appendix~\ref{sec:appendixC}, to give the following bound,
	\begin{align*}
	\twonorm{\M_{S_3} - \expec{\M_{S_3}}} \leq \delta \twonorm{\xo}^2,
	\end{align*}
	with probability at least $1-\frac{1}{m}$, as long as $m\geq C s\log n$. Now we can use the fact that $\hat{S} \subset S_3$, so that,
	\begin{align*}
	\twonorm{\M_{\hat{S}} - \expec{\M_{\hat{S}}}} \leq \twonorm{\M_{S_3} - \expec{\M_{S_3}}} \leq \delta \twonorm{\xo}^2.
	\end{align*}
	Since $\M_{\hat{S}}$ can be seen as a perturbation of $\expec{\M_{\hat{S}}}$, where the top two singular values of $\expec{\M_{\hat{S}}}$ are spaced $2\twonorm{\xShat}^2$ apart, we can use the Sin-Theta theorem \cite{sinetheta} to bound the difference between the normalized top-singular vectors $\xin$ of $\M_{\hat{S}}$ and $\x_{\hat{S}}$ of $\expec{\M_{\hat{S}}}$ as,
	%\begin{align*}
	%\twonorm{\frac{\xin \xin^\top}{\phi^2} - \frac{\xShat \xShat^\top}{\twonorm{\xShat}}^2 } \leq \frac{\tilde{\delta_{2}}\twonorm{\xo}^2}{2\twonorm{\xShat}^2 - \tilde{\delta_{2}}\twonorm{\xo}^2} \leq \bar{\delta_2}.
	%\end{align*}
	\begin{align*}
	\distop{\xin}{\xShat} &\leq \frac{\delta\twonorm{\xo}^2}{2\twonorm{\xo}^2} = \frac{\delta}{2}. \\
	\Rightarrow \min\rbrak{\twonorm{\xin-\xShat},\twonorm{\xin+\xShat}} &\leq \sqrt{2-\sqrt{4-{\delta}^2}},\\ 
	&\leq \delta_2.
	\end{align*}	
	Hence, with probability greater than $1-\frac{8}{m}$, Lemma \ref{lem:xin} holds.
\end{proof}

Combining Lemmas \ref{lem:xShat} and \ref{lem:xin}, we have the final result:
\begin{align*}
\distop{\xin}{\xo} &= \min\rbrak{{\twonorm{\xin - \xo},\twonorm{\xin + \xo}}}, \\&\leq \delta_0 \twonorm{\xo},
%\implies \distop{\xin}{\xo} &\leq c_0
\end{align*}
as long as the number of measurements $m$ follow the bound in \eqref{eq:mbound}. Hence the initial vector $\xin$ is upto a constant factor away from the true vector $\xo$. The constant  $\delta_0 \leq \delta_1 + \delta_2$ can be decreased by increasing the number of samples (see equation \eqref{eq:constm}). This completes the proof of Theorem \ref{thm:main_initial}.
%Appendix A-2
\section{Block CoPRAM initialization} \label{sec:appendixA2}
In this section we state the proofs related to the \textit{initialization} for Block CoPRAM in Algorithm~\ref{algo:initial_blk}, for block sparse signals.\\

We prove \cref{thm:block_initial} for the initialization stage of Block CoPRAM as follows. 

\blockinit*

\begin{proof}
	
Evaluating the marginals $M_{{j_b}{j_b}}$, for all $j_b \in S_b^c$,
	from \eqref{eq:margin_ub}, with probability greater than $1-\frac{5}{m}$, we have:
	\begin{align}  \label{eq:sep2}
	M_{{j_b}{j_b}} \leq \rbrak{1+11 \sqfr{\log mn}{m}} \sqrt{b}\twonorm{\xo}^2.
	\end{align}
	
Evaluating the block marginals $M_{{j_b}{j_b}}$, for $j_b \in S_b$, we use a modification of \eqref{eq:conc1}, with probability less than $ \exp\rbrak{-\frac{mt^2}{192\twonorm{\xo}^4}} \leq \frac{1}{mn}$,
	\begin{align*} 
	\frac{1}{m}\sum_{i=1}^m -X_i &\leq -t\\
	\frac{1}{m}\sum_{i=1}^m y_i^2 a_{ij}^2 - \rbrak{\twonorm{\xo}^2 + 2x_j^{*2}}   &\leq -t  
	\end{align*}
	Rearranging the terms,
%	\begin{align*}
%	\sum_{j\in j_b} M_{jj}^2 &\leq \sum_{j\in j_b} \sbrak{\rbrak{\twonorm{\xo}^2 - t} + 2x_j^{*2}}^2,\\
%	&\leq b\rbrak{\twonorm{\xo}^2 - t}^2 \\
%	&+ 4\twonorm{\x_{j_b}^{*}}^4 + 4\sqrt{b} \twonorm{\x_{j_b}^{*}}^2 \rbrak{\twonorm{\xo}^2 - t},\\
%	\implies M_{{j_b}{j_b}} &\leq \sqrt{b}\rbrak{\twonorm{\xo}^2 - t} + 2\twonorm{\x_{j_b}^*}^2,
%	\end{align*}
		\begin{align*}
	\sum_{j\in j_b} M_{jj}^2 &\leq \sum_{j\in j_b} \sbrak{\rbrak{\twonorm{\xo}^2 - t} + 2x_j^{*2}}^2,\\
	&\leq b\rbrak{\twonorm{\xo}^2 - t}^2 + 4\twonorm{\x_{j_b}^{*}}^4 + 4\sqrt{b} \twonorm{\x_{j_b}^{*}}^2 \rbrak{\twonorm{\xo}^2 - t},\\
	\implies M_{{j_b}{j_b}} &\leq \sqrt{b}\rbrak{\twonorm{\xo}^2 - t} + 2\twonorm{\x_{j_b}^*}^2,
	\end{align*}
	where the final expression holds with probability less than $\frac{b}{mn}$. Here, we have used he shorthand $\twonorm{\x_{j_b}^*}^2 \equiv \sum_{j\in j_b} x_j^{*2}$.   Finally, taking a minimum over all such block marginals $j_b \in S_b$, with probability greater than $1-\frac{1}{m}$,
	\begin{align*}
	M_{{j_b}{j_b}} &\geq \sqrt{b}\rbrak{\twonorm{\xo}^2 - t} + 2\twonorm{\x_{j_b}^*}^2, \\
	&\geq \sqrt{b}\twonorm{\xo}^2  + \twonorm{\x_{b_{min}}^*}^2,
	\end{align*}
	if $\sqrt{b}t = \twonorm{\x_{b_{min}}^*}^2 \equiv \min_{j_b\in S_b} \twonorm{\x_{j_b}^*}^2$. Assuming that $\twonorm{\x_{b_{min}^*}}^2 = \frac{C}{k}\twonorm{\x^*}^{2}$, the following holds
	\begin{align} \label{eq:sep1}
	\min_{j_b\in S_b} M_{{j_b}{j_b}} &\geq  \rbrak{1+\frac{C}{\sqrt{b}k}}\sqrt{b}\twonorm{\xo}^2.
	\end{align}
	Equating the expression for probability, 
	\begin{align*} 
	m &\geq 192 \frac{\twonorm{\xo}^4}{t^2}\log{mn}, \\
	& \geq C bk^2 \log{mn} = C\frac{s^2}{b}\log{mn},
	\end{align*}
	which puts a bound on the block marginals for $j_b \in S_b$.\\
	
	Hence, as long as $m \geq C\frac{s^2}{b}\log{n}$, there is a clear separation in the marginals,
	using \eqref{eq:sep1} and \eqref{eq:sep2},
	\begin{align*}
	\min_{j_b\in S_b} M_{{j_b}{j_b}} &\geq  \rbrak{1+\frac{C}{\sqrt{b}k}}\sqrt{b}\twonorm{\xo}^2, \\&> \rbrak{1+11\sqfr{\log mn}{m}} \sqrt{b} \twonorm{\xo}^2,\\ &\geq \max_{j_b\in S_b^c} M_{{j_b}{j_b}} , 
	\end{align*}
	where $C$ is large enough. Given that there is a clear separation in the marginals, the block support $\hat{S_{b}}$ as picked up as in Algorithm~\ref{algo:initial_blk}, is exactly the true block support $S_b$. \\
	
	It is then straightforward to show that the top singular vector of the truncated covariance matrix $\M_{\hat{S_b}}$ is actually close to the true block sparse vector $\xo$, which holds with probability greater than $1-\frac{1}{m}$.\\

Thus far, the proof requires an assumption on $\twonorm{\x_{b_{min}}^*}$. We do away with this assumption as follows:
	
	For evaluating block marginals $M_{{j_b}{j_b}}$ for $j_b \in S_b^c$, we can use the result of Lemma~\ref{lem:smallmarginals}, to obtain the same bound as in \eqref{eq:sep2}, with probability greater than $1-\frac{5}{m}$,
	\begin{align*}
	M_{{j_b}{j_b}} \leq \rbrak{1 + 11\sqrt{\frac{\log{mn}}{m}}}\sqrt{b}\twonorm{\xo}^2.
	\end{align*}
	
	For evaluating block marginals $M_{{j_b}{j_b}}$ for $j_b \in S_b$ we can use equations \eqref{eq:largeelements} and \eqref{eq:smallelements}, and extend this model of signal supports to block supports defined as:
	\begin{align*} 
	S_{b-} &= \cbrak{j_b \in S_b \mid  \twonorm{\x_{j_b}^*}^2 \equiv \sum_{j\in j_b}x_j^{*2} \leq 15\sqfr{b\log{mn}}{m}\twonorm{\xo}^2} ,\\
	S_{b+} &= \cbrak{j_b \in S_b \mid \twonorm{\x_{j_b}^*}^2 \equiv \sum_{j\in j_b}x_j^{*2} > 15\sqfr{b\log{mn}}{m}\twonorm{\xo}^2} .
	\end{align*}
	Using equation \eqref{eq:cconc}, and LHS of \eqref{eq:signalpower}, 
	\begin{align*}
	M_{jj}  &\geq 2{x_j^*}^2 - 17 \sqfr{\log mn}{m} \twonorm{\xo}^2 + \phi^2,\\
	&\geq 2{x_j^*}^2 + \rbrak{ 1 - 19 \sqfr{\log mn}{m}} \twonorm{\xo}^2.
	\end{align*}
	Constructing block marginals as $ M_{{j_b}{j_b}} \equiv \sqrt{\sum_{j\in j_b} M_{jj}^2} $,
	\begin{align*}
	M_{{j_b}{j_b}} &\geq \sqrt{b}\rbrak{ 1 - 19 \sqfr{\log mn}{m}} {\twonorm{\xo}^2} + 2\twonorm{\x_{j_b}^*}^2 ,\\
	\implies M_{{j_b}{j_b}}&\geq \rbrak{1 + 11\sqfr{b\log mn}{m}}\twonorm{\xo}^2.
	\end{align*}
 We can then extend the proof of Lemma~\ref{lem:xShat} to give the partitions,
	\begin{align*}
	\xSbhat &= \xSbpls + \x_{S_1}^* + \x_{S_2}^*,
	\\
	\xo &= \xSbmin + \xSbpls.
	\end{align*}
	and the inequalities:
	\begin{align*}
	\twonorm{\xo - \xSbhat}^2 &\leq 2\twonorm{\xSbmin }^2,\\
	&= 2\sum_{j_b\in S_{b-}} \twonorm{\x_{j_{b}}}^2,\\
	&\leq 15 k \sqfr{b\log{mn}}{m}\twonorm{\xo}^2 \leq \delta \twonorm{\xo}^2 .
	\end{align*} 
	This inequality gives us a bound on the number of measurements $m$, similar to \eqref{eq:constm},
	\begin{align*}
	m \geq \frac{15^2}{{\delta}^2} k^2 b \log mn = C(\delta) \frac{s^2}{b} \log mn ,
	\end{align*}
	with probability greater than $1-\frac{7}{m}$. This gives us the evaluation of block-marginals for $j_b \in S_b$ and $S_b^c$, respectively. It is then straightforward to show that the top singular vector of the truncated covariance matrix $\M_{\hat{S_{b}}}$, given $\hat{S_{b}}$ is actually close to the true block sparse vector $\xo$ with probability greater than $1-\frac{1}{m}$.
\end{proof}

%Appendix B - proofs of descent
\section{CoPRAM and Block CoPRAM descent}\label{sec:appendixB}

In this section we state the proofs related to the \textit{descent to optimal solution} in Algorithm~\ref{algo:copram} (CoPRAM), for sparse signals and Algorithm~\ref{algo:altminblock} (Block CoPRAM), for block sparse signals. This includes the proof of Theorem~\ref{thm:convergence} and Theorem~\ref{thm:blockconvergence}. We prove \cref{thm:convergence} to show descent of the CoPRAM algorithm, as follows.

\textbf{Note:} For evaluation of the distance measure $\distop{\cdot}{\cdot}$, we only consider $\distop{\xt}{\xo} = \dist{\xt}{\xo}$, assuming that $\distop{\xin}{\xo} = \dist{\xin}{\xo}$ at the end of initialization stage. We claim that wlog, the same results would hold, if $\distop{\xin}{\xo} = \twonorm{\xin+\xo}$.

\convergence*

\begin{algorithm}[!h]
	\caption{CoSaMP}
	\label{algo:cosamp}
	\begin{algorithmic}[1]
		\INPUT $\Phi = \frac{\A}{\sqrt{m}}, \uu = \frac{\P^{t}\y}{\sqrt{m}},s,\x^t$.
		\STATE Initialize 
		\begin{align*}
		\x^{t+1,0} &\leftarrow \x^t \quad \text{initialize to best possible estimate}\\
		\rr &\leftarrow \uu \quad\:\: \text{residue}\\
		l &\leftarrow 0 \quad \:\:\text{CoSaMP internal counter}
		\end{align*}
		\WHILE{halting condition not true,}
		\STATE	
		\begin{align*}
		l &\leftarrow l+1 \\
		\vv &\leftarrow \Phi^T \rr \quad \text{signal proxy}\\
		\Omega &\leftarrow \supp{\vv_{2s}}\\
		\Gamma &\leftarrow \Omega \cup \supp{\x^{t+1,l-1}}\\
		\w &\leftarrow \Phi^{\dagger}_\Gamma \uu \quad \text{corresponding to }\:\Gamma, 0\:\text{elsewhere}\\
		\x^{t+1,l} &\leftarrow \text{Truncate to top}\: s \:\text{values of}\: \w, \text{call this support} \:\Gamma_s\\
		\rr &\leftarrow \uu - \Phi \x^{t+1,l}
		\end{align*}
		\ENDWHILE
		\STATE	$\x^{t+1,L} \leftarrow \Phi_{\Gamma_s}^\dagger u$.
		\OUTPUT $\xtplus \leftarrow \x^{t+1,L}$
	\end{algorithmic}
\end{algorithm}

To show the descent of our alternating minimization algorithm using CoSaMP, we need to analyze the reduction in error, per step of CoSaMP, (refer Algorithm~\ref{algo:cosamp}) first:
\begin{align} \nonumber
\twonorm{\x^{t+1,l+1} - \xo} &= \twonorm{\x^{t+1,l+1} - \w + \w - \xo},\\
&\leq 2\twonorm{\xo - \w} \label{eq:stepred}
\end{align}
where $\w$ corresponds to the $\ell$'th run of CoSaMP for the $(t+1)^{th}$ update of $\x$. Using RIP of $\Phi = \frac{\A}{\sqrt{m}}$,
\begin{align} \label{eq:cosamp_steperr}
\twonorm{\x^{t+1,l+1} - \xo} \leq \frac{2}{\sqrt{1-\delta_s}}\twonorm{\Phi \xo - \Phi \w},
\end{align}
with high probability, where $\delta_s$ is the RIP constant. Now, analyzing the inputs to CoSaMP, in step 4 of Algorithm~\ref{algo:copram},
\begin{align}
\uu &= \frac{\P^t \y }{\sqrt{m}}, \nonumber\\
&= \sign{\A\x^t}\circ \frac{\abs{\A\xo}}{\sqrt{m}}, \nonumber\\
&= \sign{\Phi\xt } \circ \cbrak{\rbrak{\Phi \xo} \circ{\sign{\Phi \xo}} }, \nonumber\\
&= \Phi \xo + \rbrak{\sign{\Phi \xt} \pm \sign{\Phi \xo}}\circ \Phi \xo, \nonumber\\
\implies \uu - \Phi \xo &= \pm\rbrak{\sign{\Phi \xt} - \sign{\Phi \xo}} \circ \Phi \xo, \label{eq:err_phase}\\&= E_{ph},  \nonumber
\end{align}
where $E_{ph} \equiv \rbrak{\sign{\Phi \xt} \pm \sign{\Phi \xo}} \circ \Phi \xo $, is error due to failure in estimating the correct phase.\\

Using equation \eqref{eq:err_phase} and substituting into equation \eqref{eq:cosamp_steperr}, the per-step reduction in error for each run of CoSaMP is:
%\begin{align*}
%&\twonorm{\x^{t+1,l+1} - \xo} \hspace{2in}\\ 
%&\leq \frac{2}{\sqrt{1-\delta_s}}\twonorm{\uu - E_{ph}  - \Phi \w}, \\
%&\leq \frac{2}{\sqrt{1-\delta_s}}\twonorm{\uu - \Phi \w} + \frac{2}{\sqrt{1-\delta_s}}\twonorm{E_{ph}},\\
%&\leq \frac{2}{\sqrt{1-\delta_s}}\twonorm{\uu - \Phi_{\Gamma} \w_{\Gamma}} + \frac{2}{\sqrt{1-\delta_s}}\twonorm{E_{ph}},\\
%&\leq \frac{2}{\sqrt{1-\delta_s}}\twonorm{\uu - \Phi_{\Gamma} \x_{\Gamma}^*} + \frac{2}{\sqrt{1-\delta_s}}\twonorm{E_{ph}},\\
%&\leq \frac{2}{\sqrt{1-\delta_s}}\twonorm{\Phi\xo + E_{ph} - \Phi_{\Gamma} \x_{\Gamma}^*} + \frac{2}{\sqrt{1-\delta_s}}\twonorm{E_{ph}},\\
%&\leq \frac{2}{\sqrt{1-\delta_s}}\twonorm{\Phi\xo  - \Phi_{\Gamma} \x_{\Gamma}^*} + \frac{4}{\sqrt{1-\delta_s}}\twonorm{E_{ph}},\\
%&\leq \frac{2}{\sqrt{1-\delta_s}}\twonorm{\Phi_{\Gamma^c}\x_{\Gamma^c}^* }+ \frac{4}{\sqrt{1-\delta_s}}\twonorm{E_{ph}},\\
%&\leq 2\sqfr{1+\delta_s}{1-\delta_s}\twonorm{\rbrak{{\xo-\x^{t+1,l}}}_{\Gamma^c} }+ \frac{4}{\sqrt{1-\delta_s}}\twonorm{E_{ph}},\\
%&= \rho_1 \twonorm{\rbrak{{\xo-\x^{t+1,l}}}_{\Gamma^c} } + \rho_2\twonorm{E_{ph}},
%\end{align*}
\begin{align*}
\twonorm{\x^{t+1,l+1} - \xo} 
&\leq \frac{2}{\sqrt{1-\delta_s}}\twonorm{\uu - E_{ph}  - \Phi \w}, \\
&\leq \frac{2}{\sqrt{1-\delta_s}}\twonorm{\uu - \Phi \w} + \frac{2}{\sqrt{1-\delta_s}}\twonorm{E_{ph}},\\
&\leq \frac{2}{\sqrt{1-\delta_s}}\twonorm{\uu - \Phi_{\Gamma} \w_{\Gamma}} + \frac{2}{\sqrt{1-\delta_s}}\twonorm{E_{ph}},\\
&\leq \frac{2}{\sqrt{1-\delta_s}}\twonorm{\uu - \Phi_{\Gamma} \x_{\Gamma}^*} + \frac{2}{\sqrt{1-\delta_s}}\twonorm{E_{ph}},\\
&\leq \frac{2}{\sqrt{1-\delta_s}}\twonorm{\Phi\xo + E_{ph} - \Phi_{\Gamma} \x_{\Gamma}^*} + \frac{2}{\sqrt{1-\delta_s}}\twonorm{E_{ph}},\\
&\leq \frac{2}{\sqrt{1-\delta_s}}\twonorm{\Phi\xo  - \Phi_{\Gamma} \x_{\Gamma}^*} + \frac{4}{\sqrt{1-\delta_s}}\twonorm{E_{ph}},\\
&\leq \frac{2}{\sqrt{1-\delta_s}}\twonorm{\Phi_{\Gamma^c}\x_{\Gamma^c}^* }+ \frac{4}{\sqrt{1-\delta_s}}\twonorm{E_{ph}},\\
&\leq 2\sqfr{1+\delta_s}{1-\delta_s}\twonorm{\rbrak{{\xo-\x^{t+1,l}}}_{\Gamma^c} }+ \frac{4}{\sqrt{1-\delta_s}}\twonorm{E_{ph}},\\
&= \rho_1 \twonorm{\rbrak{{\xo-\x^{t+1,l}}}_{\Gamma^c} } + \rho_2\twonorm{E_{ph}},
\end{align*}
where the first step is from using triangle inequality, the second step is from using the fact that $\w$ is exactly $3s$-sparse with support $\Gamma$. The third step is using the fact that truncation of $\w$ in $\Gamma, \in \reals^{3s}$ , is the minimizer of the LS problem $\argmin_{\x\in \reals^{3s}} \twonorm{\Phi_{\Gamma}\x - \uu}$, the fourth step uses \eqref{eq:err_phase} again, the final step uses RIP again (which holds with probability greater than $1-e^{-\gamma_1 m}$, with $\gamma_1$ being a positive constant).\\

Finally, the first term in the previous inequality can be bounded using (Lemma 4.2 of CoSaMP \cite{cosamp}, refer Lemma~\ref{lem:cosamp_err_red}), to yeild,
%\begin{align*}
%&\twonorm{\x^{t+1,l+1} - \xo}\\ &\leq \rho_1\rho_3\twonorm{{\xo-\x^{t+1,l}} } +  \rbrak{\rho_1\rho_4+\rho_2}\twonorm{E_{ph}},
%\end{align*}
\begin{align*}
\twonorm{\x^{t+1,l+1} - \xo} &\leq \rho_1\rho_3\twonorm{{\xo-\x^{t+1,l}} } +  \rbrak{\rho_1\rho_4+\rho_2}\twonorm{E_{ph}},
\end{align*}
where $\rho_3,
\rho_4$ are as stated in Lemma~\ref{lem:cosamp_err_red}. Assuming that CoSaMP is let to run a maximum of $L$ iterations,
%\begin{align} \nonumber
%&\twonorm{\x^{t+1} - \xo}\\ &\leq (\rho_1\rho_3)^L\twonorm{{\xo-\x^{t}} } \nonumber\\ \nonumber
%&+ (\rho_1\rho_4+\rho_2)\rbrak{1+\rho_1\rho_3 \dots (\rho_1\rho_3)^{L-1}}\twonorm{E_{ph}},\\ 
%&\leq (\rho_1\rho_3)^L\twonorm{{\xo-\x^{t}} } + \frac{(\rho_1\rho_4+\rho_2)}{\rbrak{1-\rho_1\rho_3}}\twonorm{E_{ph}} . \label{eq:cosamp_prefinal}
%\end{align}
\begin{align} \nonumber
\twonorm{\x^{t+1} - \xo} &\leq (\rho_1\rho_3)^L\twonorm{{\xo-\x^{t}} } + (\rho_1\rho_4+\rho_2)\rbrak{1+\rho_1\rho_3 \dots (\rho_1\rho_3)^{L-1}}\twonorm{E_{ph}},\\ 
&\leq (\rho_1\rho_3)^L\twonorm{{\xo-\x^{t}} } + \frac{(\rho_1\rho_4+\rho_2)}{\rbrak{1-\rho_1\rho_3}}\twonorm{E_{ph}} . \label{eq:cosamp_prefinal}
\end{align}
The second part of this proof requires a bound on the phase error term $\twonorm{E_{ph}}$:
\begin{align*}
E_{ph} = \pm\rbrak{\sign{\Phi \xt} - \sign{\Phi \xo}} \circ \Phi \xo .
\end{align*}

We proceed to finish this proof by invoking Lemma~\ref{lem:phase_err_bound}.

\begin{lemma} \label{lem:phase_err_bound}
	As long as the initial estimate is a small distance away from the true signal $\xo \in \mathcal{M}_s$, 
	\begin{align*}
	\distop{\xin}{\xo} \leq \delta_0 \twonorm{\xo},
	\end{align*}
	and subsequently,
	\begin{align*}
		\distop{\xt}{\xo} \leq \delta_0 \twonorm{\xo},
	\end{align*}
	then the following bound holds,
%	\begin{align*}
%	&\frac{2}{m} \sum_{i=1}^{m} \abs{\ai^T\xo}^2 \ones_{\cbrak{ (\ai^T\xt)(\ai^T\xo)\leq 0}}\\ 
%	&\leq \frac{2}{(1-\delta_0)^2}\rbrak{\delta + \sqfr{21}{20}\delta_0}\twonorm{\xt - \xo}^2.
%	\end{align*}
	\begin{align*}
\frac{2}{m} \sum_{i=1}^{m} \abs{\ai^T\xo}^2 \ones_{\cbrak{ (\ai^T\xt)(\ai^T\xo)\leq 0}} \leq \frac{2}{(1-\delta_0)^2}\rbrak{\delta + \sqfr{21}{20}\delta_0}\twonorm{\xt - \xo}^2,
\end{align*}
	with probability greater than $1-e^{-\gamma_2 m}$, where $\gamma_2$ is a positive constant, as long as $m > Cs\log \frac{n}{s}$. We can use this to bound the phase error as,
	\begin{align*}
	\twonorm{E_{ph}} \leq \rho_5 \twonorm{\xt - \xo},
	\end{align*}
	where $\rho_5 = \frac{\sqrt{2}}{(1-\delta_0)}\sqrt{\delta + \sqfr{21}{20}\delta_0}$, $\delta \approx 0.001$.
\end{lemma}
This proof has been adapted from Lemma 7.19 of \cite{mahdi} and uses the generic chaining techniques of~\cite{talagrand,dirksen}. Using this in addition to equation \eqref{eq:cosamp_prefinal}, we have our final per-step error reduction for a single run of CoPRAM (Algorithm~\ref{algo:copram}), as:
\begin{align} \nonumber
\twonorm{\xtplus - \xo} &\leq \rbrak{(\rho_1\rho_3)^L + \rho_5\frac{(\rho_1\rho_4+\rho_2)}{\rbrak{1-\rho_1\rho_3}}} \twonorm{\xt - \xo},\\ \label{eq:descent_final}
&\leq \rho_0 \twonorm{\xt - \xo},
\end{align}
where $\rho_0 < 1$.\\

\textbf{Evaluating convergence parameter $\rho_0$:}. To obtain per-step reduction in error, we require $\rho_0 < 1$. For sake of numerical analysis, $\delta_s , \delta_{2s}, \delta_{4s} \leq 0.0001$, then $\rho_1 \approx 1, \rho_3 \approx 0.0002 $. Let $\delta_0 = 0.012$, then $\rho_5 \approx 0.16$. Similarly, $\rho_2 \approx 4$ and $\rho_4 \approx 2$. Suppose CoSaMP is allowed to run for $L=5$ iterations then, $\rho_0 \approx 0.96 < 1$.

The inequalities used for CoSaMP, particularly \eqref{eq:stepred} can be made tighter, which would give less tight restrictions on the factor $\delta_0$, that controls how close the intial estimate is to the true signal $\xo$.\\

We now restate \cref{thm:blockconvergence} for Block CoPRAM as follows.

\blockconvergence*

The proof for this is a natural extention to the one we have proved in Theorem~\ref{thm:convergence}, and would use the results from the paper on model-based compressive sensing \cite{modelcs}, wherever Block CoSaMP is invoked.
%% analysis for noise %%

\section{Noise robustness} \label{sec:AppendixE}
In this section, we show that both CoPRAM and Block CoPRAM are robust to noise, and establish the proof of Theorem \ref{thm:noise}. 

\noise*

We assume the modified version of \eqref{eqn:magnitude-measurements} acquisition model:
\begin{align}
\tilde{\y} = \abs{\A\xo} + \epsilon = \y + \epsilon
\end{align}
where $\epsilon$ is distributed according to a scaled sub-exponential random variable. If the variance of the noise is much smaller in comparison to the magnitude of the measurements, then the following approximation holds:
\begin{align}
\tilde{y_i}^2 = y_i^2 + \eta_i
\end{align}
where $\eta_i$ are sub-exponentail random variables (special case is Gaussian with distribution $\gauss(\mathbf{0},\sigma^2)$).

\subsection{CoPRAM Initialization}
The marginals used for initialization will get modified as:
\begin{align}
\tilde{M}_{jj} = M_{jj} + \frac{1}{m}\sum_{i=1}^{m} \eta_i a_{ij}^2
\end{align}
Similarly signal power gets modified as:
\begin{align}
\tilde{\phi}^2  = \phi^2 + \frac{1}{m}\sum_{i=1}^m \eta_i
\end{align}

Then much of the analysis for the initialization, follows from 
\cite{cai}. Key points in the proof stated in Appendix \ref{sec:appendixA} get modified as follows:

\textbf{Lemma} \eqref{lem:smallmarginals}
\begin{align*}
\tilde{M}_{jj} - \tilde{\phi}^2 = \rbrak{{M}_{jj} - {\phi}^2} + \rbrak{\frac{1}{m}\sum_{i=1}^m \eta_i (a_{ij}^2 -1)}
\end{align*}
where the second term is bounded as Equation (6.4) in \cite{cai} as
\begin{align*}
\max_{1 \leq j \leq n} \abs{{\frac{1}{m}\sum_{i=1}^m \eta_i (a_{ij}^2 -1)}} \leq C\sigma \sqrt{\frac{\log mp}{m}}
\end{align*}
with probability greater than $1-\frac{4}{m}$
and the first term is bounded using \eqref{eq:marginal_ub}.

Since we have assumed that the noise variance is much lesser than the signal power, $C\sigma \leq \alpha \twonorm{\xo}$. Hence, the new marginal threshold for $j\in S^c$ is:
\begin{align*}
\tilde{M}_{jj} - \tilde{\phi}^2 &\leq \rbrak{8 + \alpha}\sqrt{\frac{\log {mp}}{m}} \twonorm{\xo}^2\\
\Rightarrow \tilde{M}_{jj} &\leq \rbrak{1+(11+\alpha)\sqrt{\frac{\log {mn}}{n}}}\twonorm{\xo}^2 = \tilde{\Theta}
\end{align*}
with probability greater than $1-\frac{9}{m}$.

\textbf{Lemma} \eqref{lem:largemarginals}
Similarly, $\tilde{M}_{jj} - \tilde{\phi}^2$ is lower bounded for $j\in S_+$, using \eqref{eq:cconc} along with Equation (6.4) in \cite{cai}:
\begin{align*}
\tilde{M}_{jj} \geq \tilde{\Theta}
\end{align*}
where we redefine $S_+$ and $S_-$ as
\begin{align*} 
S_{+} = \cbrak{j \in S \mid  x_j^{*2} > \rbrak{15+\frac{\alpha}{2}}\sqfr{\log{mn}}{m}\twonorm{\xo}^2} .
\end{align*}
Subsequently, we can define $S_{-}$ as
\begin{align*} 
S_{-} = \cbrak{j \in S \mid  x_j^{*2} \leq \rbrak{15+\frac{\alpha}{2}}\sqfr{\log{mn}}{m}\twonorm{\xo}^2} ,
\end{align*}

Lemmas \eqref{lem:xShat} and \eqref{lem:xin} hold using these modifications, upto a constant factor. Theorem \ref{thm:main_initial} holds with probability greater than $1-\frac{12}{m}$.

\subsection{CoPRAM Descent}

In the Descent stage, apart from the \textit{phase estimation} error and \textit{signal estimation} error, we also have an additional measurement error term. This modification reflects in Equation \eqref{eq:err_phase} of Appendix \ref{sec:appendixB} as:
\begin{align*}
\uu - \Phi \xo 
&= \pm\rbrak{\sign{\Phi \xt} - \sign{\Phi \xo}} \circ \Phi \xo \pm \sign{\Phi \xo} \circ \epsilon, \\
&= E_{ph} + E_{m},
\end{align*}
where $E_{m} \equiv \pm\sign{\Phi \xo} \circ \epsilon$ is the measurement error. This error propagates to modify Equation \eqref{eq:cosamp_prefinal} as:
%\begin{align*}
%&\twonorm{\x^{t+1} - \xo}\\ 
%&\leq (\rho_1\rho_3)^L\twonorm{{\xo-\x^{t}} } + \frac{(\rho_1\rho_4+\rho_2)}{\rbrak{1-\rho_1\rho_3}}\twonorm{E_{ph} + E_m}, \\
%&\leq  (\rho_1\rho_3)^L\twonorm{{\xo-\x^{t}} } + \frac{(\rho_1\rho_4+\rho_2)}{\rbrak{1-\rho_1\rho_3}}\twonorm{E_{ph}}
%+ \frac{(\rho_1\rho_4+\rho_2)}{\rbrak{1-\rho_1\rho_3}}\twonorm{\epsilon}.
%\end{align*}
\begin{align*}
\twonorm{\x^{t+1} - \xo} &\leq (\rho_1\rho_3)^L\twonorm{{\xo-\x^{t}} } + \frac{(\rho_1\rho_4+\rho_2)}{\rbrak{1-\rho_1\rho_3}}\twonorm{E_{ph} + E_m}, \\
&\leq  (\rho_1\rho_3)^L\twonorm{{\xo-\x^{t}} } + \frac{(\rho_1\rho_4+\rho_2)}{\rbrak{1-\rho_1\rho_3}}\twonorm{E_{ph}}
+ \frac{(\rho_1\rho_4+\rho_2)}{\rbrak{1-\rho_1\rho_3}}\twonorm{\epsilon}.
\end{align*}
Finally, the main convergence result from Equation \eqref{eq:descent_final} gets modified as:
\begin{align} \nonumber
\twonorm{\xtplus - \xo} \leq \rho_0 \twonorm{\xt - \xo} + \rho_m \twonorm{\epsilon},
\end{align}
where $\rho_m = \frac{(\rho_1\rho_4+\rho_2)}{\rbrak{1-\rho_1\rho_3}} < 1$.
After a total of $t_o$ iterations of CoPRAM, we get the convergence result:
\begin{align*}
\twonorm{\x^{t_o} - \xo} &\leq \rho_0^{t_o} \twonorm{\xin - \xo} + \frac{\rho_m}{1-\rho_0} \twonorm{\epsilon},\\
&\leq \rho_0^{t_o} \delta_0 \twonorm{\xo} + \frac{\rho_m}{1-\rho_0} \twonorm{\epsilon} \approx 0.5\twonorm{\epsilon}.
\end{align*}
Thus the quality of reconstruction depends on level of input noise $\epsilon$, which is bounded with high probability.
%% analysis on power-law decay type of signals

\section{Power law decay} \label{sec:AppendixPL}

In this section we state the proof of the sample complexity required for sparse signals following power-law decay, as described in Theorem \ref{thm:PLdecay}. 
\PLdecay*
Note that a power-law decaying signal can be normalized as follows:
\begin{align} \nonumber
{x_j^*}^2 &\leq \frac{C(\alpha)}{j^\alpha},\\ \nonumber
\twonorm{\xo}^2 &\leq C(\alpha) \sum_{j=1}^{s} \frac{1}{j^\alpha} := C(\alpha) \zeta(s,\alpha),\\ \nonumber
\twonorm{\xo}^2 + e &= C(\alpha) \zeta(s,\alpha),\\ \nonumber
\implies C(\alpha) &= \frac{\twonorm{\xo}^2 + e}{\zeta(s,\alpha)},\\
\label{eq:Cnorm}
\implies C(\alpha) &= \frac{\zeta(s,0)}{\zeta(s,\alpha)} C(0) = \frac{s}{\zeta(s,\alpha)} C(0),
\end{align}
where $e$ is an indicator of the tightness of power-law inequality.

Additionally, we can find the index at which all power-law decaying signal elements fall bellow threshold $\Theta_0$: 
\begin{align} \label{eq:jb}
j' = \left\lfloor{\rbrak{\frac{C(\alpha)}{\Theta_0}}^{\frac{1}{\alpha}} }\right\rfloor.
\end{align} 
The task is to find a tighter bound for $\twonorm{\xo_{S_-}}$, as compared to \eqref{eq:loosebound}. This bound can be established using our additional assumption of power law decay. We analyze this bound using the threshold criteria $\Theta_0 = 15\sqrt{\frac{\log mn}{m}} \twonorm{\xo}^2$, as in \eqref{eq:largeelements} and $\alpha > 1$ :
\vspace{-0.5cm}
\begin{comment}
\textbf{Case 1}: $ 0 \leq \alpha < 1$
\begin{gather}
\twonorm{\xo_{S_-}}^2 \leq j' \Theta_0 + C(\alpha)\sum_{j=j'+1}^{s}\frac{1}{j^\alpha} \leq C(\alpha)\sum_{j=1}^{s}\frac{1}{j^\alpha} \leq C(\alpha) \int_{0^+}^{s} j^{-\alpha} dj \leq \frac{C(\alpha)}{1-\alpha}\rbrak{s^{1-\alpha} - (0^+)^{1-\alpha}} = \frac{C(\alpha)}{1-\alpha} s^{1-\alpha} \label{eq:sc1}
\end{gather}
Additionally, we know that for $\alpha = 0$, our previous analysis, holds, therefore:
\begin{align*}
\frac{C(0)}{1-0} s^{1-0} &\leq \Theta_0 s \\
\implies C(0) &\leq \Theta_0
\end{align*}
This bound holds for all $0 < \alpha < 1$, as long as $s \geq (1-\alpha)^{-1/\alpha}$, which generally holds, unless $\alpha \to 1^-$. The bound in \eqref{eq:sc1} then gets modified as
\begin{align*}
	\twonorm{\xo_{S_-}}^2 \leq  \frac{\Theta_0}{1-\alpha}  \frac{s}{\zeta(s,\alpha)} s^{1-\alpha}
\end{align*}
using \eqref{eq:Cnorm} and the fact that $\zeta(s,\alpha)^{-1} \leq 1$.
This yields us the expression for sample complexity, with the bound in \eqref{eq:constm} getting modified as
\begin{gather*}
\twonorm{\xo_{S_-}}^2  \leq \frac{15}{1-\alpha}\sqrt{\frac{\log mn}{m}} \twonorm{\xo}^2 \twonorm{\xo}^2 s^{1-\alpha} \leq \delta_1^2\twonorm{\xo}^2 \\
\implies m \geq C(\delta_1, \alpha) s^{2(1-\alpha)} \log mn 
\end{gather*}
\textbf{Case 2}: $ \alpha = 1$
\begin{gather*}
\twonorm{\xo_{S_-}}^2  \leq C\sum_{j=1}^{s}\frac{1}{j} \leq C + C \int_{1}^{s} \frac{1}{j} dj = C(1 + \log s) \leq \Theta_0 (1+\log s)  \\
\implies m \geq C(\delta_1) \log^2 s \log mn 
\end{gather*}
\end{comment}

\begin{align*}
\twonorm{\xo_{S_-}}^2 &\leq j' \Theta_0 + C\sum_{j=j'+1}^{s}\frac{1}{j^\alpha}, \\
 &\leq j' \Theta_0 + C \int_{j'}^{s} j^{-\alpha} dj, \\
&\leq j' \Theta_0 + \frac{C}{\alpha-1}\rbrak{\frac{1}{{j'}^{\alpha-1}} - \frac{1}{s^{\alpha-1}}},\\ &\leq j' \Theta_0 + \frac{C}{\alpha-1}\frac{1}{{j'}^{\alpha-1}} = \frac{\alpha}{\alpha-1} j' \Theta_0, \\&= \frac{\alpha}{\alpha-1}\rbrak{ \frac{\twonorm{\xo}^2 + e}{\zeta(s,\alpha)}\frac{1}{15\twonorm{\xo}^2}\sqrt{\frac{m}{\log mn}}}^{1/\alpha} \Theta_0,\\
&\lessapprox \frac{\alpha}{\alpha-1} \zeta(s,\alpha)^{-1/\alpha} \Theta_0,\\
&\leq \frac{\alpha}{\alpha-1} \Theta_0, \\
\implies m &\geq C(\delta_1, \alpha) \log mn,
\end{align*}
where we have used \eqref{eq:jb} and the fact that $\zeta(s,\alpha)^{-1/\alpha} \leq 1$.

\begin{comment}
\begin{align*}
\twonorm{\xo_{S_-}}^2 \leq \sum_{j=1}^{s} x_1^2 \rbrak{\frac{1}{j}}^\alpha \leq \frac{{x_1^*}^2}{\alpha - 1} \rbrak{1 - \frac{1}{s^{\alpha-1}}} +  {x_1^*}^2 \leq {x_1^*}^2\rbrak{\lim_{\alpha\to 1^+} \frac{1-s^{1-\alpha}}{\alpha - 1}} + {x_1^*}^2 = {x_1^*}^2 \rbrak{ \log s + 1}
%\leq C {x_1^*}^2 \sqrt{s} \leq C \rbrak{15\sqrt{\frac{\log mn}{m}}} \sqrt{s}
\end{align*}
This is further bounded by the condition imposed by Case 1,
%for $C$ large enough. 
and hence the bound in \eqref{eq:constm} gets modified as 
\begin{gather*}
15\sqrt{\frac{\log mn}{m}} (1 + \log s) \twonorm{\xo}^2 < \: (loose) \: 15\sqrt{\frac{\log mn}{m}} \sqrt{s} \twonorm{\xo}^2 \leq \delta_1^2\twonorm{\xo}^2\\
\implies m \geq \frac{900C^2}{\delta_1^2} (\log s)^2 \log mn = C(\delta_1)  (\log s)^2 \log mn 
\end{gather*}
or more loosely,
\begin{align*}
m \geq \frac{900C^2}{\delta_1^2} s \log mn = C(\delta_1)  (\log s)^2 \log mn 
\end{align*}
\textbf{Case 2}: $x_1^2 \geq \Theta_0$
\begin{align*}
\twonorm{\xo_{S_-}}^2 \leq (j'-1)\Theta_0 + \sum_{j=j'+1}^{s} {x_1^*}^2 \rbrak{\frac{1}{j}}^\alpha 
\end{align*}
\end{comment}

In this regime, the sample complexity for the overall algorithm is dominated by the sample complexity for the descent stage ($\order{s \log n}$ = max($C \log mn$,$C s \log n/s$)).
%\clearpage
%Appendix C - supplementary lemmas for proofs in A and B
%Appendix B
\section{Supplementary theorems} \label{sec:appendixC}

In this section we state some of the lemmas with or without proofs, used in Appendices~\ref{sec:appendixA} and \ref{sec:appendixB}.\\

\begin{lemma} \label{lem:signalpower}
%	With probability of at least $1-\frac{2}{m}$,
%	\begin{align}
%	\label{eq:signalpower2}
%	1 - 2\sqrt{\frac{\log{m}}{m}} \leq \frac{\phi^2}{\twonorm{\xo}^2} \leq 1 + 2\sqrt{\frac{\log{m}}{m}} + {\frac{2\log{m}}{m}}.
%	\end{align}
	With probability of at least $1-\frac{1}{m}$,
	\begin{align}
	\label{eq:signalpower}
	 \rbrak{1 - 2\sqrt{\frac{\log{m}}{m}} }{\twonorm{\xo}^2} \leq {\phi^2} \leq \rbrak{1 + 3\sqrt{\frac{\log{m}}{m}} }{\twonorm{\xo}^2}.
	\end{align}
\end{lemma}

%\begin{proof}
%	The squared-signal power follows a chi-squared distribution of degree $m$.
%	\begin{align*}
%	\frac{m\phi^2}{\twonorm{\xo}^2} = \frac{\sum_{i=1}^m   (\ai^\top\xo)^2}{\twonorm{\xo}^2} = \chi^2(m).
%	\end{align*} 
%	Using Lemma 4.1 of \cite{chisquare}, Lemma \ref{lem:signalpower} holds.
%\end{proof}
\begin{proof}
	Rotational invariance property of Gaussian distributions imply that $\y_i^2 \equiv (\sum_{j=1}^n a_{ij}x_j^*)^2$ has the same distrubution as $a_{ij}^2\twonorm{\xo}^2$. Using Lemma 4.1 of \cite{chisquare} on $a_{ij}^2$, we can obtain the upper bound,
	\begin{align*}
	&\prob{\frac{1}{m}\sum_{i=1}^{m} a_{ij}^2 - 1 \geq 2\frac{\sqrt{m\log m}}{m} + 2\frac{\log m}{m}}\\
	&\leq \exp{\rbrak{-\log m}} = \frac{1}{m}.
	\end{align*}
	Similarly, we can obtain the lower bound,
	\begin{align*}
	\prob{\frac{1}{m}\sum_{i=1}^{m} a_{ij}^2 - 1 \leq -2\frac{\sqrt{m\log m}}{m}}\leq \exp{\rbrak{-\log m}} = \frac{1}{m}.
	\end{align*}
	The signal power $\phi^2$ is then bounded from below as
	\begin{align*}
	\rbrak{1-2\sqfr{\log m}{m} }\twonorm{\xo}^2 &\leq \phi^2,
	\end{align*}
	and similarly, it is bounded from above as,
	\begin{align*} 
	\phi^2 &\leq \rbrak{1+2\sqfr{\log m}{m} + 2\frac{\log m}{m}}\twonorm{\xo}^2,\\ &< \rbrak{1+3\sqfr{\log m}{m}} \twonorm{\xo}^2,
	\end{align*}
	with probability at least $1-\frac{1}{m}$, for $m>C$, large enough. If $m\approx 1000$, then the bounds are,
	\begin{align*} 
	(1-\delta)\twonorm{\xo}^2 \leq \phi^2 \leq (1+\delta)\twonorm{\xo}^2,
	\end{align*} 
	where $\delta = 0.0207$.
\end{proof}
\begin{lemma} 
	\label{lem:svd_mat}
	With probability at least $1-\frac{1}{m}$, the following holds,  
	\begin{gather*}
	\twonorm{\frac{1}{m} \sum_{i=1}^m \abs{\ai^\top_{S_3}\xo}^2\ai_{S_3}\ai_{S_3}^\top - \rbrak{\twonorm{\xo}^2\I_{S_3}+2\xo\xo^\top}}\\
	\leq \delta \twonorm{\xo}^2
	\end{gather*}
	where $\card{S_3} \leq 2s$, provided $m > C(\delta)(2s)\log(2s)$. 
\end{lemma}
This proof has been adapted from Lemma A.6 of \cite{cai}.

\begin{lemma}
	\label{lem:martingale}
	Suppose $X_1 \dots X_m$ are i.i.d. centered, bounded real-valued random variables obeying
	\begin{align*}
	X_i &\leq b,\\
	\expec{X_i} &= 0,\\
	\expec{X_i^2} &= v^2,\\
	\sigma^2 &= \max\cbrak{b^2,v^2},
	\end{align*}
	with cumulative distribution function of the standard normal distribution being denoted as
	\begin{align*}
	\Phi(x) &= \int_{-\infty}^{x} \phi(t) dt,\\
	\phi(t) &= \frsq{1}{2\pi} \exp\rbrak{-\frac{t^2}{2}},
	\end{align*}
	then 
%	\begin{align*}
%	&\prob{\sum_{i=1}^m X_i \geq t}\\ \leq &\min\cbrak{\exp\rbrak{\frac{-t^2}{2\sigma^2}},25\rbrak{1-\Phi\rbrak{\frac{t}{\sigma}}}}.
%	\end{align*}
		\begin{align*}
	\prob{\sum_{i=1}^m X_i \geq t} &\leq \min\cbrak{\exp\rbrak{\frac{-t^2}{2\sigma^2}},25\rbrak{1-\Phi\rbrak{\frac{t}{\sigma}}}}.
	\end{align*}
	This establishes the tail probability of martingale with differences bounded from one side \cite{martingale}.
\end{lemma}

\begin{lemma}
	\label{lem:cosamp_err_red}
	The $2s$-sparse residual error $\twonorm{(\xo - \x^{t+1,l})_{\Gamma^c}}$ can be upper bounded as,
	\begin{align*}
	\twonorm{(\xo - \x^{t+1,l})_{\Gamma^c}} &\leq \twonorm{(\xo - \x^{t+1,l})_{\Omega^c}}\\ 
	&\leq \rho_3 \twonorm{(\xo - \x^{t+1,l})} + \rho_4\twonorm{E_{ph}} 
	\end{align*} 
	where $\rho_3 = \frac{\delta_{2s}+\delta_{4s}}{1-\delta_{2s}}$ and $\rho_4 = \frac{2\sqrt{1+\delta_{2s}}}{1-\delta_{2s}}$.
\end{lemma}
This lemma has been adapted from Lemmas 4.2 and 4.3 of \cite{cosamp}.

\bibliographystyle{unsrt}
\bibliography{biblio_nips,chinbiblio}

% use section* for acknowledgment

% Can use something like this to put references on a page
% by themselves when using endfloat and the captionsoff option.
\ifCLASSOPTIONcaptionsoff
  \newpage
\fi

% trigger a \newpage just before the given reference
% number - used to balance the columns on the last page
% adjust value as needed - may need to be readjusted if
% the document is modified later
%\IEEEtriggeratref{8}
% The "triggered" command can be changed if desired:
%\IEEEtriggercmd{\enlargethispage{-5in}}

% references section

% can use a bibliography generated by BibTeX as a .bbl file
% BibTeX documentation can be easily obtained at:
% http://mirror.ctan.org/biblio/bibtex/contrib/doc/
% The IEEEtran BibTeX style support page is at:
% http://www.michaelshell.org/tex/ieeetran/bibtex/
%\bibliographystyle{IEEEtran}
% argument is your BibTeX string definitions and bibliography database(s)
%\bibliography{IEEEabrv,../bib/paper}
%
% <OR> manually copy in the resultant .bbl file
% set second argument of \begin to the number of references
% (used to reserve space for the reference number labels box)

% biography section
% 
% If you have an EPS/PDF photo (graphicx package needed) extra braces are
% needed around the contents of the optional argument to biography to prevent
% the LaTeX parser from getting confused when it sees the complicated
% \includegraphics command within an optional argument. (You could create
% your own custom macro containing the \includegraphics command to make things
% simpler here.)
%\begin{IEEEbiography}[{\includegraphics[width=1in,height=1.25in,clip,keepaspectratio]{mshell}}]{Michael Shell}
% or if you just want to reserve a space for a photo:

% if you will not have a photo at all:
\begin{IEEEbiographynophoto}{Gauri Jagatap}
is a PhD student in the Department of Electrical and Computer Engineering at Iowa State University, where she works with Dr. Chinmay Hegde. Her area of research comprises of developing fast and efficient algorithms for  machine learning and signal processing applications. Prior to ISU, she completed dual degrees, with a Bachelor's in Electrical and Electronics Engineering and Master's in Physics from BITS Pilani University (India), in 2015.
\end{IEEEbiographynophoto}
% insert where needed to balance the two columns on the last page with
% biographies
%\newpage
\begin{IEEEbiographynophoto}{Chinmay Hegde} (S,'07, M, '12) is an assistant professor, and Black and Veatch Faculty Fellow, in the Department of Electrical and Computer Engineering at Iowa State University. His research focuses on developing fast and robust algorithms for machine learning and statistical signal processing, with applications to imaging problems. Prior to this, he was a Shell-MIT Postdoctoral Fellow in CSAIL at the Massachusetts Institute of Technology. He is the recipient of several awards, including best paper awards at SPARS and ICML, a best poster award at MMLS, the Budd Award for Best Engineering Thesis at Rice University in 2013, and the NSF CRII Award in 2016.
\end{IEEEbiographynophoto}

% You can push biographies down or up by placing
% a \vfill before or after them. The appropriate
% use of \vfill depends on what kind of text is
% on the last page and whether or not the columns
% are being equalized.

%\vfill

% Can be used to pull up biographies so that the bottom of the last one
% is flush with the other column.
%\enlargethispage{-5in}

% that's all folks
\end{document}